%% file: main.tex
\documentclass[twoside,11pt]{article}

\usepackage{blindtext}

\usepackage[preprint]{jmlr2e}

\usepackage{booktabs}       
\usepackage{xcolor}
\usepackage{subcaption}

\input{math_commands}

\usepackage{lastpage}
\jmlrheading{26}{2025}{1-\pageref{LastPage}}{5/25; Revised}{}{21-0000}{Eric Aubinais, Elisabeth Gassiat and Pablo Piantanida}

\ShortHeadings{Fundamental Limits of MIAs on ML Models}{Aubinais, Gassiat and Piantanida}
\firstpageno{1}
\begin{document}

\title{Fundamental Limits of Membership Inference Attacks on Machine Learning Models}

\author{\name Eric Aubinais \email eric.aubinais@universite-paris-saclay.fr \\
       \addr Université Paris-Saclay, CNRS,\\ Laboratoire de mathématiques d’Orsay, 91405, Orsay, France
       \AND
       \name Elisabeth Gassiat \email elisabeth.gassiat@universite-paris-saclay.fr \\
       \addr Université Paris-Saclay, CNRS,\\ Laboratoire de mathématiques d’Orsay, 91405, Orsay, France
       \AND Pablo Piantanida \email pablo.piantanida@mila.quebec \\ 
       \addr ILLS - International Laboratory on Learning Systems, \\ MILA - Quebec AI Institute, Montreal (QC), Canada, \\ CNRS, CentraleSupélec | Université Paris-Saclay \\
       }

\editor{My editor}

\maketitle

\input{macrov3}

\begin{abstract}
\input{Sections/v3/Abstract-v3}

\end{abstract}


\section{Introduction}
\input{Sections/v3/I-Introduction-v3}

\section{Background and Problem Setup}
\label{prob_form}
\input{Sections/v3/II-ProblemSetup-v3}

\section{Performance Assessment of Membership Inference Attacks}
\label{main_res_perf}
\input{Sections/v3/III-Performanceassessment-v3}
\section{Overfitting Causes Lack of Security}
\label{overfitting}
\input{Sections/v3/VII-Overfitting-v3}

\section{Security is Data Size Dependent}
\label{main_res_mean}
\label{main_res_disc}
\input{Sections/v3/IV-V-Empiricalmeananddiscretecase-v3}

\section{Numerical Experiments}
\label{exp}
\input{Sections/v3/VIII-numericalexp-v3}

\section{Summary and Discussion}
\label{ccl}
\input{Sections/v3/VI-commentsandremarks-v3}

\begin{acks}
Elisabeth Gassiat is supported by Institut Universitaire de France, and by the Agence Nationale de la Recherche under projects ANR-21-CE23-0035-02 and ANR-23-
CE40-0018-02.\\

\noindent We sincerely thank the reviewers for their valuable and constructive feedback.
\end{acks}

\newpage 
\appendix
\section{More comments on Section \ref{prob_form}}
\label{sec:comments-setup}
\input{Appendix/v3/G-CommentsonSetup-v3}

\section{More comments on Overfitting}
\label{sec:overfitting}
\input{Appendix/v3/I-CommentonOverfitting-v3}

\section{More comments on Section \ref{main_res_disc}}
\label{sec:cp}
\input{Appendix/v3/F-MorecommentsonCP-v3}

\section{Setting for Section \ref{exp}}
\label{sec:exp}
\input{Appendix/v3/J-Numexpsetup-v3}


\section{Proofs of Section \ref{main_res_perf}}
\label{sec:proof:performances}
\input{Appendix/v3/A-proofforPerformances-v3}

\section{Proofs of Section \ref{overfitting}}
\label{sec:proof:overfitting}
\input{Appendix/v3/D-proofforOverfitting-v3}

\section{Proofs of Section \ref{main_res_disc}}
\label{sec:proof:means-discrete}
\input{Appendix/v3/B-proofforMeans-Discrete-v3}

\section{Differential Privacy and MIS}
\label{sec:dp_vs_mis}
\input{Appendix/v3/Y-DPrelation-v3}



\vskip 0.2in
\bibliography{bibliography}

\end{document}

%% file: math_commands.tex
\usepackage{amsmath}
\usepackage{amssymb, bm}
\usepackage{mathtools}

\def\eqref#1{equation~\ref{#1}}

\def\1{\bm{1}}

\def\rv{{\textnormal{v}}}

\def\rx{{\mathnormal{x}}}
\def\ry{{\mathnormal{y}}}
\def\rz{{\mathnormal{z}}}

\def\rvz{{\mathbf{z}}}

\def\va{{\bm{a}}}

\def\vz{{\bm{z}}}

\def\mI{{\bm{I}}}

\DeclareMathAlphabet{\mathsfit}{\encodingdefault}{\sfdefault}{m}{sl}
\SetMathAlphabet{\mathsfit}{bold}{\encodingdefault}{\sfdefault}{bx}{n}

\def\gA{{\mathcal{A}}}

\def\gD{{\mathcal{D}}}

\def\gF{{\mathcal{F}}}
\def\gG{{\mathcal{G}}}
\def\gH{{\mathcal{H}}}

\def\gL{{\mathcal{L}}}
\def\gM{{\mathcal{M}}}
\def\gN{{\mathcal{N}}}

\def\gP{{\mathcal{P}}}

\def\gX{{\mathcal{X}}}
\def\gY{{\mathcal{Y}}}
\def\gZ{{\mathcal{Z}}}

\def\sN{{\mathbb{N}}}

\def\sS{{\mathbb{S}}}

\def\sX{{\mathbb{X}}}
\def\sY{{\mathbb{Y}}}

\newcommand{\E}{\mathbb{E}}

\newcommand{\R}{\mathbb{R}}



%% file: macrov3.tex
\def\deltan{\Delta_n(P,\gA)}
\def\deltaniid{\Delta_n^{\text{non-\textit{i.i.d.}}}(P,\gA)}
\def\deltanu{\delta_{\nu}\left(\gL((\hat{\theta}_n,\rz_0)), \gL((\hat{\theta}_n,\rz_1))\right)}

\def\mis{\text{Sec}_{\nu,\lambda,n}(P,\gA)}

\def\acc{\text{Acc}_{\nu,\lambda,n}(\phi;P,\gA)}

\def\dnu{\Delta_{\nu,\lambda,n}(P,\gA)}

\def\dnut{\tilde D_\nu\left(\mathbb{P}_{\left(\hat{\theta}_n,\rz_1\right)}, \mathbb{P}_{\left(\hat{\theta}_n,\rz_0\right)} \right)}

\def\lb{\left(1-\frac{1}{\gamma}\right)_+}
\def\gl{{\gamma}}
\def\fgl{\gamma}
\def\lg{{1/\gamma}}
\def\flg{\frac{1}{\gamma}}
\def\dt{\tilde{D}_\gl\left(\mathbb{P}_{(\hat{\theta}_n,\rz_0)},\mathbb{P}_{(\hat{\theta}_n,\rz_1)}\right)}

\def\finv#1{\frac{1}{#1}}
\def\inv#1{\left({#1}\right)^{-1}}

%% file: Sections/v3/Abstract-v3.tex
Membership inference attacks (MIA) can reveal whether a particular data point was part of the training dataset, potentially exposing sensitive information about individuals. 
This article provides theoretical guarantees by exploring the fundamental statistical limitations associated with MIAs on machine learning models at large. More precisely, we first derive the  statistical quantity that governs the effectiveness and success of such attacks. We then theoretically prove that in a  non-linear regression setting with overfitting learning procedures, attacks may have a high probability of success. Finally, we investigate several situations for which we provide bounds on this quantity of interest.
Interestingly, our findings indicate that discretizing the data might enhance the learning procedure's security. Specifically, it is demonstrated to be limited by a constant, which quantifies the diversity of the underlying data distribution. We illustrate those results through simple simulations.

%% file: Sections/v3/I-Introduction-v3.tex
In today's data-driven era, machine learning models are designed to reach higher performance, and the size of new models will then inherently increase, therefore the information stored (or memorized) in the parameters~\citep{hartley2022measuring,boudingInfo2023}. The protection of sensitive information is of paramount importance. Membership Inference Attacks (MIAs) have emerged as a concerning threat, capable of unveiling whether a specific data point was part of the training dataset of a machine learning model~\citep{shokri2017membership,10.1145/3133956.3134077,8835245,NEURIPS2019_60a6c400}. Such attacks can potentially compromise individual privacy and security by exposing sensitive information~\citep{extracting2023}. Furthermore, the recent publication by \citet{unknown} from the National Institute of Standards and Technology (NIST) explicitly highlights that a membership inference attack (MIA) which successfully identifies an individual as being part of the dataset used to train a target model constitutes a breach of confidentiality. This raises a crucial question: How should we evaluate and certify privacy in machine learning models?

\noindent To date, the most comprehensive defense mechanism against privacy attacks is Differential Privacy (DP), a framework initially introduced by \citet{dwork2006calibrating}. DP has shown remarkable adaptability in safeguarding the privacy of machine learning models during training, as demonstrated by the works of \citet{10.5555/3361338.3361469,Hannun2021MeasuringDL}. However, it is worth noting that achieving a high level of privacy through differentially private training often comes at a significant cost to the accuracy of the model, especially when aiming for a low privacy parameter~\citep{bayesMIA2019}. Conversely, when evaluating the practical effectiveness of DP in terms of its ability to protect against privacy attacks empirically, the outlook is considerably more positive. DP has demonstrated its efficacy across a diverse spectrum of attacks, encompassing MIAs, attribute inference, and data reconstruction (see~\citet{pmlr-v202-guo23e} and references therein). DP has been extensively used to understand the performances of MIAs against learning systems~\cite{thudi2022bounding} or how a mechanism could be introduced to defend oneself against MIAs ~\cite{he2022doctor, izzo2022provable}.

\noindent Empirical evidence suggests that small models compared to the size of training set are often sufficient to thwart the majority of existent threats and empirically summarized in~\citet{baluta2022causal}. Similarly, when the architecture of a machine learning model is overcomplex with respect to the size of the training set, model overfitting increases the effectiveness of MIAs, as has been identified by~\citet{shokri2017membership, yeom2018privacy, he2022doctor, boudingInfo2023}. However, despite these empirical findings, there remains a significant gap in our theoretical understanding of this phenomenon. 
Most of the existing literature on MIAs focuses on developing and analyzing specific, albeit efficient, attack strategies. While these contributions have been crucial in highlighting the privacy risks associated with MIAs, they fall short when it comes to auditing the privacy risks of ML models comprehensively. The failure of specific attacks does not guarantee that others will not succeed, meaning that individual attack strategies alone cannot provide a full certification of privacy in ML systems. 

\noindent In this article, we explore the theoretical statistical principles underlying the privacy limitations of learning procedures in machine learning systems on a broad scale. Our investigation commences by establishing the \textbf{fundamental statistical quantity that governs the effectiveness and success of MIA attacks.} In the learning model we are examining, our primary focus is on learning procedures that are characterized as functions of the empirical distribution of their training data. Specifically, we concentrate on datasets of independent and identically distributed ({\it i.i.d.}) samples. To evaluate the feasibility and limitations of MIAs on a learning procedure, we will assess their \textbf{accuracy} by measuring the weighted probability of successfully determining membership. Notably, we assess the security of a learning procedure based on the highest level of accuracy achieved among all theoretically feasible MIAs. To this end, we delve into the intricacies of MIA and derive insights into the key factors that influence its outcomes. Subsequently, we explore various scenarios: overfitting learning procedures, empirical mean-based learning procedures and discrete data, among others, presenting bounds on this central statistical quantity

\subsection{Contributions}
\label{subsec:contr}
In our research, we make theoretical contributions to the understanding of MIAs on machine learning models. Our key contributions can be summarized as follows:

\begin{itemize}
    \item  \textbf{Identification of the Central} Statistical Quantity: We introduce the critical statistical quantity denoted as $\dnu$, where $\nu$ is the theoretical fraction of \textit{non-member} samples, $\lambda$ is the importance accorded to the error of type I, $n$ represents the size of the training dataset, $P$ is the data distribution, and $\gA$ is the underlying learning procedure. This quantity plays a pivotal role in assessing the feasibility and limitations of MIAs. We show that the quantity $\dnu$ is an $f-$divergence, which provides an intuitive measure of how distinct parameters of a model can be with respect to a sample in the training set, and as a result, it indicates the extent to which we can potentially recover sample membership through MIAs. Consequently, we demonstrate that when $\dnu$ is small, the accuracy of the best MIA is notably constrained. Conversely, when $\dnu$ approaches $1$, the best MIA is successful with high probability. This highlights the importance of $\dnu$ in characterizing information disclosure in relation to the training set.

    \item \textbf{Lower Bounds for Overfitting Learning Procedures:}
    For learning procedures that overfit with high probability, we exhibit a lower bound on $\dnu$ (see Theorem \ref{thm:overfitting}). In a (non-)linear regression setting, we further theoretically demonstrate that learning procedures for which small training loss is reached, loss-based MIAs can achieve almost perfect inference, as illustrated in Section \ref{exp} by numerical experiments. Up to our knowledge, this is the first theoretical proof that overfitting indeed opens the way to successful MIAs.

    \item \textbf{Precise Upper Bounds for Empirical Mean-Based Learning Procedures:} For learning procedures that compute functions of empirical means, we establish upper bounds on $\dnu$. We prove that $\dnu)$ is bounded from above by a constant, determined by $(P, \mathcal{A})$, multiplied by $n^{-1/2}$. In practical terms, this means that having $\Omega(\varepsilon^{-2})$ samples in the dataset is sufficient to ensure that $\dnu$ remains below $\varepsilon$ for any $\varepsilon\in(0,1)$. 
    
    
    \item \textbf{Maximization of $\dnu$:} In scenarios involving discrete data (e.g., tabular data sets), we provide a precise formula for maximizing $\dnu$ across all learning procedures $\mathcal{A}$. Additionally, under specific assumptions, we determine that this maximization is proportional to $n^{-1/2}$ and to a  quantity $C_K(P)$ which measures the diversity of the underlying data distribution. Interestingly, this result highlights the inherent properties of certain datasets that make them more vulnerable to MIAs, regardless of the underlying learning procedure. We illustrate this behavior with numerical experiments in Section \ref{exp}.



\end{itemize}

\noindent The objective of the paper is therefore to highlight the central quantity of interest $\dnu$ governing the success of MIAs and propose an analysis in different scenarios. 




\subsection{Related Works}

\textbf{Privacy Attacks.}  The majority of cutting-edge attacks follow a consistent approach within a framework known as Black-Box. In this framework, where access to the data distribution is available, attacks assess the performance of a model by comparing it to a group of ``shadow models". These shadow models are trained with the same architecture but on an artificially and independently generated dataset from the same data distribution. Notably, loss evaluated on training samples are expected to be much lower than when evaluated on ``test points". Therefore, a  significant disparity between these losses indicates that the sample in question was encountered during the training, effectively identifying it as a member. This is intuitively related to some sort of ``stability" of the algorithm on training samples \citep{bousquetstability}. Interestingly, we explicitly identify the exact quantity controlling  the accuracy of effective MIAs which may be interpreted as a measure of stability of the underlying algorithm. In fact, as highlighted by \cite{rezaei2021difficulty}, it is important to note that MIAs are not universally effective and their success depends on various factors. These factors include the characteristics of the data distribution, the architecture of the model, particularly its size, the size of the training dataset, and others, as discussed recently by \cite{shokri2017membership, carlini2022membership}. Subsequently, there has been a growing body of research delving into Membership Inference Attacks (MIAs) on a wide array of machine learning models, encompassing regression models \citep{Gupta2021MembershipIA}, generation models \citep{Hayes2018MI}, and embedding models \citep{10.1145/3372297.3417270}. A comprehensive overview of the existing body of work on various MIAs has been systematically compiled in a thorough survey conducted by  \cite{hu2022membership}. 
While studies of MIAs through DP already reveal precise bounds, it is worth noting that these induce a significant loss of performance on the learning task. 
It is worth to emphasize that our work does not focus on differential privacy. We further discuss the relation between our quantity $\dnu$ and DP mechanisms in Section \ref{sec:dp_vs_mis}.
Interestingly, the findings of the Section \ref{main_res_mean} reveal a threshold on the minimum number of training samples to overcome the need of introducing DP mechanisms. 



\noindent \textbf{Overfitting Effects.}   The pioneering work by \cite{shokri2017membership} has effectively elucidated the relationship between overfitting and the privacy risks inherent in many widely-used machine learning algorithms. These empirical studies clearly point out that overfitting can often provide attackers with the means to carry out membership inference attacks. This connection is extensively elaborated upon by  \cite{salem2018ml, yeom2018privacy}, and later by \cite{he2022doctor}, among other researchers. Overfitting tends to occur when the underlying model has a complex architecture or when there is limited training data available, as explained in \cite{baluta2022causal}. Recent works~\citep{yeom2018privacy, boudingInfo2023} investigated the theoretical aspects of the overfitting effect on the performances of MIAs,  showing that the MIA performances can be lower bounded by a function of the \textit{generalization gap} under some assumptions on the loss function.
In our paper, we explicitly emphasize these insights by quantifying the dependence of $\dnu$ either on the dataset size and underlying structural parameters, or explicitly on the overfitting probability of the learning model.


\noindent \textbf{Memorization Effects.} Machine learning models trained on private datasets may inadvertently reveal sensitive data due to the nature of the training process. This potential disclosure of sensitive information occurs as a result of various factors inherent to the training procedure, which include the extraction of patterns, associations, and subtle correlations from the data~\citep{10.1145/3133956.3134077,10.1145/3446776}. While the primary objective  is to generalize from data and make predictions, there is a risk that these models may also pick up on, and inadvertently expose, confidential or private information contained within the training data. This phenomenon is particularly concerning as it can lead to privacy breaches, compromising the confidentiality and security of personal or sensitive data  \citep{hartley2022measuring, NEURIPS2022_564b5f82, carlini2019secret, leino2020stolen, Thomas2020InvestigatingTI}. Recent empirical studies have shed light on the fact that, in these scenarios, it is relatively rare for the average data point to be revealed by  learning models \citep{tirumala2022memorization, murakonda2007ml, song2017machine}. What these studies have consistently shown is that it is the outlier samples that are more likely to undergo memorization by the model \citep{feldman2020does}, leading to potential data leakage. This pattern can be attributed to the nature of learning algorithms, which strive to generalize from the data and make predictions based on common patterns and trends. Average or typical data points tend to conform to these patterns and are thus less likely to stand out. On the other hand, outlier samples, by their very definition, deviate significantly from the norm and may capture the attention of the model. So when an outlier sample is memorized, it means the model has learned it exceptionally well, potentially retaining the unique characteristics of that data point. As a consequence, when exposed to similar data points during inference, the model may inadvertently leak information it learned from the outliers, compromising the privacy and security of the underlying data. An increasing body of research is dedicated to the understanding  of memorization effects in language models \citep{carlini2023quantifying}. In the context of our research, it is important to highlight that our primary focus is on understanding the accuracy of MIAs but not its relationship with memorization. Indeed, this  connection remains an area of ongoing exploration and inquiry in our work.





%% file: Sections/v3/II-ProblemSetup-v3.tex
In this paper, we focus on MIAs, the ability of recovering membership to a training dataset 
${\rvz\coloneqq(\rz_1,\cdots, \rz_n)\in\gZ^n}$ of a test point $\Tilde{\rz}\in\gZ$ from a 
predictor $\hat{\mu} = \mu_{\hat{\theta}_n}$ in a model $\gF\coloneqq\{\mu_{\theta} : 
\theta\in\Theta\}$, where $\Theta$ is the space of parameters. The predictor is identified to 
its parameters $\hat{\theta}_n\in\Theta$ learned from $\rvz$ through a  \textbf{learning procedure} ${\gA
: \bigcup_{k>0}\gZ^k\to\gP'\subseteq \gP(\Theta)}$, that is $\hat{\theta}_n$ follows the 
distribution $\gA(\rvz)$ conditionally to $\rvz$, which we assume we have access to. Here, 
$\gP(\Theta)$ is the set of all distributions on $\Theta$, and $\gP'$ is the range of $\gA$.

\noindent This means that there exists a function $g$ and a random variable $\xi$ independent of $\rvz$ such that $\hat{\theta}_n = g(\rvz, \xi)$.
When $\gA$ takes values in the set of Dirac distributions, that is $\hat{\theta}_n$ is a deterministic function of the data, we shall identify the parameters directly to the output of the learning procedure $\hat{\theta}_n \coloneqq \gA(\rz_1,\cdots,\rz_n)$. \\ 
Throughout the paper, we will further assume that $\gA$ can be expressed as a function of the empirical distribution of the training dataset. Letting $\gM$ be the set of all discrete distributions on $\gZ$, and $\hat{P}_n$ be the empirical distribution of the training dataset, it means that there exists a (randomized) function $G:\gM\to\gP'$  such that we have $\gA(\rz_1,\cdots,\rz_n) = G(\hat{P}_n)$ (almost surely). \\
Interestingly, if a learning procedure minimizes an empirical cost, then it satisfies this assumption. In particular, maximum likelihood based learning procedures or Bayesian methods from \citet{bayesMIA2019} are special cases. Any instance of a learning procedure in what follows will satisfy these assumptions. We further discuss this assumption in Appendix \ref{sec:comments-setup}.

\noindent Within this framework, we consider MIAs as functions of the parameters and the test point whose outputs are $0$ or $1$.

\begin{definition}[Membership Inference Attack - MIA] 
\label{def:mia}
Any measurable map $\phi:\Theta\times\gZ\to\{0,1\}$ 
is called a \textbf{Membership Inference Attack}.
\end{definition}

\noindent Hereinafter, we assume that MIAs might access more information, including randomization, the learning procedure $\gA$ and/or the distribution of the data $P$. This framework is usually referred to as white-box \citep{hu2022membership}.


\noindent We encode 
membership to the training data set as $1$. We assume that 
$\rz_1,\ldots,\rz_n$ are independent and identically distributed 
(\textit{i.i.d.}) random variables with distribution $P$. Following 
\cite{boudingInfo2023} and \cite{bayesMIA2019} frameworks, among others, we suppose that 
the test point $\Tilde{\rz}$ is to be drawn from $P$ independently from the 
samples $\rz_1,\ldots,\rz_n$ with probability $\nu\in(0,1)$. Otherwise, conditionally to $\rvz$, we set $\Tilde{\rz}$ to any $\rz_j$ each with 
uniform probability $1/n$. \\
Letting $U$ be a random variable with distribution ${\hat{P}_n\coloneqq\frac{1}{n}\sum_{j=1}^n\delta_{\rz_j}}$ conditionally to $\rvz$, $\rz_0$ to be drawn independently from $P$ and $T$ be a random variable having Bernoulli distribution with parameter $\nu$ 
and independent of any other random variables, we can state 
\begin{align*}
    {\Tilde{\rz} \coloneqq T\rz_0 + (1-T)U}.
\end{align*}

\noindent The probability $\nu$ represent the theoretical \textbf{fraction of non-member} samples when evaluating MIAs, which is usually unknown in practice. In particular, it reflects how the data are tested. \\
\noindent A sensible choice to evaluate the performance of an MIA would be to consider its probability of successfully guessing the membership, i.e.
\begin{align*}
    \mathbb{P}(\phi(\hat{\theta}_n,\tilde\rz) =1-T) &= \mathbb{P}(\phi(\hat{\theta}_n,\rz_0)=0, T=1) + \mathbb{P}(\phi(\hat{\theta}_n,U)=1, T=0) \\&= \mathbb{E}\left[\texttt{TNR} + \texttt{TPR}\right],
\end{align*}
\noindent where TPR (resp. TNR) is the True Positive Rate (resp. True Negative Rate) of the MIA $\phi$. However, in an MIA setting, the TPR is arguably more important than the TNR. Therefore, we define the \textbf{importance of the TPR} as a real number $\lambda>0$, and we measure the performance of an MIA by its weighted probability of successfully guessing the membership of the test point.

\begin{definition}[Accuracy of an MIA] The \textbf{accuracy of an MIA} $\phi$ is defined as 
\begin{equation}
{\text{Acc}}_n(\phi; P,\gA ) \coloneqq \mathbb{P}\left (\phi(\hat{\theta}_n,\rz_0)= 0, T=1\right ) + \lambda\mathbb{P}\left(\phi(\hat{\theta}_n,U)=1, T=0\right),\label{eq-acc-metric}
\end{equation}
where the probability is taken over all randomness.
\end{definition}

\noindent We have the following remarks: 
\begin{itemize} 
\item The accuracy of an MIA scales from $0$ to $\nu+\lambda(1-\nu)$. Constant MIAs $\phi_0\equiv0$ and $\phi_1\equiv1$ have respectively an accuracy equal to $\nu$ and $\lambda(1-\nu)$, which means that we can always build an MIA with accuracy of at least $ \max(\nu, \lambda(1-\nu))$
and any MIA performing worse than this quantity is irrelevant to use. Particularly, this means that we have $$\max(\nu,\lambda(1-\nu))\leq\sup_\phi\;\acc\leq \nu + \lambda(1-\nu).$$
Moreover, the probability in the definition of the accuracy is taken over the randomness of the learning procedure and the data. This means that $\acc$ measures the accuracy of an MIA over the learning procedure $\gA$ and the task $P$ rather than over the trained model $\mu_{\hat{\theta}_n}$.
\item  Isolated attacks are insufficient to fully certify the privacy level of ML models. Instead, controlling the optimal achievable accuracy provides a way to audit the model's privacy. Specifically, when $\sup_{\phi}\;{\text{Acc}}_n(\phi; P,\gA )$ is low, it indicates that the model is secure. Conversely, if this value is high, it indicates a potential vulnerability to successful attacks, even if specific MIAs have not been particularly successful.
\item When considering the balanced case $\lambda=1$, one shall observe that the accuracy is simply the probability of successfully guessing the membership, i.e. $\mathbb{P}\left(\phi(\hat{\theta}_n,\tilde\rz)=1-T\right)$.
\end{itemize}

\noindent We now define the \textbf{Membership Inference Security} of a learning procedure as a quantity summarizing the amount of security of the system against MIAs. 

\begin{definition}[Membership Inference Security - MIS] 
\label{def:mis}
The Membership Inference Security of a learning procedure $\gA$ is 
\begin{equation}
\label{def:sec}
\mis \coloneqq \frac{1}{\min(\nu,\lambda(1-\nu))}\left(\nu+\lambda(1-\nu)-\underset{\phi}{\sup}\;\acc\right), 
\end{equation}
where the supremum is taken over all MIAs.
\end{definition}

\noindent From the first point preceding Definition \ref{def:mis}, we see that the MIS has been defined to ensure that it scales from $0$ (the best MIA approaches perfect guess of membership) to $1$ (MIAs can not do better than $\phi_0$ and $\phi_1$).

\begin{remark}[Model-specific attack - limitations of this approach] In the framework we introduced, the MIS represents the security of the learning procedure $\gA$. One could understandably want to consider the security of a trained model, by changing the definition and considering the quantity $\sup_\phi d\left(\phi; \rz_1,\cdots,\rz_n\right)$ conditionally to the training dataset $\{\rz_1,\cdots,\rz_n\}$, for some metric $d$, where $d$ scales from $0$ to $1$. Natural choices of $d$ include $d(\phi;\rz_1,\cdots,\rz_n)=\mathbb{P}(\phi(\hat{\theta}_n,\tilde\rz)=1-T \mid \rz_1,\cdots,\rz_n)$ or $d(\phi;\rz_1,\cdots,\rz_n)= TPR(\phi;\rz_1,\cdots,\rz_n)$. Unfortunately, this problem is degenerate as for those two choices of metric $d$, it can be easily shown that,

$$\sup_\phi d\left(\phi; \rz_1,\cdots,\rz_n\right)=1,$$

\noindent for any learning procedure and any dataset. In other words, the underlying problem of studying the best performing MIA conditionally to the data has no relevant insight.\\
\noindent In this article, to avoid this degeneracy, we undertake the path of shifting the focus onto the learning procedure $\gA$ and the task $P$.
\end{remark}



%% file: Sections/v3/III-Performanceassessment-v3.tex
In this section, we highlight the \textbf{Central} Statistical Quantity for the  assessment of the accuracy of membership inference attacks, and show some basic properties on it. For any $\alpha>0$, define the function $\tilde D_\alpha$ as 
\begin{align}
  \label{eq:defdeltanu}
\tilde D_\alpha(P,Q) &\coloneqq \finv{\alpha}\underset{B}{\sup}\;\left[\alpha P(B) - Q(B)\right] \\ 
\nonumber &= \finv\alpha\underset{B}{\sup}\;\left[Q(B) - \alpha P(B)\right] + \left(1-\frac{1}{\alpha}\right),
\end{align}
\noindent and the function $D_\alpha$ as 
\begin{align}
    \label{eq:defDalpha}
    D_\alpha(P,Q) \coloneqq \max(1,\alpha)\left[\tilde D_\alpha(P,Q) - \left(1-\frac{1}{\alpha}\right)_+\right],
\end{align}
\noindent for any distributions $P$ and $Q$, where the supremum is taken over all measurable sets. Defining $\gamma\coloneqq\frac{\nu}{\lambda(1-\nu)}$, we will then show that the central statistical quantity $\dnu$ is defined as
\begin{equation}
\label{eq:delta}
\dnu = D_{\gamma}\left(\mathbb{P}_{\left(\hat{\theta}_n,\rz_0\right)}, \mathbb{P}_{\left(\hat{\theta}_n,\rz_1\right)}\right),
\end{equation}


\noindent which depends on $\nu$, $\lambda,$ $P$, $n$ and $\gA$. Here, for any random variable $\rx$, $\mathbb{P}_\rx$ denotes its distribution, and for any real number $a\in\mathbb{R}$, $a_+=\max(0,a)$. The quantity $D_\alpha$ has the remarkable property of being an $f-$divergence \citep{renyi1961measures, csiszar2004information}, which we formalize in the following proposition.
\begin{proposition}
\label{prop:dnu-props}
The map $(P,Q)\mapsto D_\alpha(P,Q)$ is an $f-$divergence between $P$ and $Q$ with as generator the function $f_{\alpha}(x) = \frac{1}{2}\max(1,\alpha)\left[|x-1/\alpha| - |1-1/\alpha|\right]$. Additionally, it holds that 
\begin{equation}
\label{eq:dnu-bounds}
0 \leq D_\alpha(P,Q) \leq 1.
\end{equation}
\noindent If $\rx_1$ and $\rx_2$ are random variables with joint distribution $\mathbb{P}_{(\rx_1,\rx_2)}$, then for any function $f \in \{D_\alpha,\tilde D_\alpha\}$ it holds that 
\begin{equation}
\label{eq:dnu-cond}
f(\mathbb{P}_{\rx_1}\otimes\mathbb{P}_{\rx_2},\mathbb{P}_{(\rx_1,\rx_2)}) = \mathbb{E}_{\rx_2}\left[f\left(\mathbb{P}_{\rx_1},\mathbb{P}_{\rx_1\mid\rx_2}\right)\right],
\end{equation}
\noindent where $\mathbb{P}_{\rx_1\mid\rx_2}$ is the distribution of $\rx_1$ conditionally to $\rx_2$. 
\end{proposition}

\noindent We have the following comments: 

\begin{itemize}
\item The quantity $\dnu$ can be interpreted as quantifying some stability of the learning procedure. Here, $\rz_1$ represents an arbitrary random sample from the training set, while $\rz_0$ denotes a random sample that has not been seen during training. Thus, our quantity $\dnu$ captures the sensitivity of the learning procedure to individual samples by quantifying the distance between the joint distribution $\mathbb{P}_{(\hat{\theta}_n,\rz_1)}$ and the product distribution $\mathbb{P}_{(\hat{\theta}_n,\rz_0)} = \mathbb{P}_{\hat{\theta}_n}\otimes\mathbb{P}_{\rz_0}$. When individual samples have little influence on the output model, replacing a single sample in the training dataset causes only a minor shift in the model’s parameters. In such cases, the output parameters exhibit minimal dependence on any individual sample, leading to a low value of $\dnu$. Conversely, if it is highly sensitive to each sample, even small changes in the dataset result in significant alterations to the parameters of the model. 


\item It is worth noting that in the ML literature, standard measures typically rely on the joint distribution between the ML model and the entire dataset (such as privacy measures, generalization bounds, etc.). In contrast, our novel approach diverges from this by focusing solely on the joint distribution between the ML model and a single training sample. Also, our metric does not measure the privacy of a trained model, it rather measures the MIA-wise privacy of a learning procedure.

\item The choice of $\rz_1$ is arbitrary and is only for simplicity purpose.

\end{itemize}

\noindent We now state the main theorem displaying the relation between $\mis$ and $\dnu$.

\begin{theorem}[Key bound on accuracy] 
\label{thm:dnu}
Suppose $P$ is any distribution and $\gA$ is any learning procedure. Then the accuracy of any MIA $\phi$ satisfies:

$$\acc - \max(\nu,\lambda(1-\nu)) \leq  \min(1,1/\gamma)\nu\dnu.$$
In particular, we have
$$
{\text{Sec}}_n(P,\gA) = 1 - \dnu.
$$

\end{theorem}

\noindent Recall that $\max(\nu,\lambda(1-\nu))$ is the maximum accuracy between the constant MIAs $\phi_0$ and $\phi_1$. The first point of Theorem \ref{thm:dnu} states that if $\dnu$ is low, then no MIA can perform substantially better than the constant MIAs. The second point of Theorem \ref{thm:dnu} shows that $\dnu$ is the quantity that controls the best possible accuracy of MIAs.\\
\noindent We see that $\dnu$ appears to be the key mathematical quantity for assessing the accuracy of MIAs.  Furthermore, it is worth to emphasize that there is no assumption on the data distribution $P$.  For instance, we can take into account outliers by making $P$ a mixture. \\
\begin{remark}[Relation with other divergences]
\label{remark:pinsker}
Proposition \ref{prop:dnu-props} shows that $D_\alpha$ is a divergence. In particular this means that it satisfies a Data Processing Inequality and is invariant by translation and rescaling. When $\alpha=1$, $D_\alpha$ coincides with the total variation distance $\|\cdot\|_{TV}$ and the inequality $D_\alpha(P,Q) \leq \max(1,\alpha)\|P-Q\|_{TV}$ holds for any $\alpha>0$. In any case, the upper bound over $D_\alpha$ is reached when the supports of the distributions are disjoint. The lower bound is reached when $P=Q$. \\
\noindent Using Pinsker's inequality, it is easy to show that we have the relation: 
\begin{equation}
    \mis \geq 1 - \max(1,\gamma)\sqrt{I(\hat{\theta}_n; \rz_1)/2},
\end{equation}
\noindent where $I(\hat{\theta}_n;\rz_1)$ is the mutual information between the parameters $\hat{\theta}_n$ and one random sample $\rz_1$. This mutual information can be interpreted as a measure of how much the parameters $\hat{\theta}_n$ memorize the information about $\rz_1$. Nevertheless, Pinsker's inequality typically yields a loose bound, indicating that relying on mutual information might be overly conservative in many cases.  
\end{remark}

\begin{remark}[Differential Privacy]
\label{remark:dp}
Interestingly, if an $(\varepsilon,\delta)-$differentially private  \citep{dwork2014algorithmic} mechanism $\mathbb{M}$ is used to secure the learning procedure $\gA$ by composition $\mathbb{M}\circ\gA$, then a bound heuristically similar to the bound obtained for the KL divergence by \cite{dwork2010boosting} can be stated. See also \cite{duchi2024right} for other metrics. Specifically, the following relation (proved in \eqref{eq:proof-remark8} in Section \ref{subsec:compToMis}) holds
\begin{equation}
    \text{Sec}_{\nu,\lambda,n}(P,\mathbb{M}\circ\gA) \geq 1 - \max(1,\gamma)\left[(e^\varepsilon-1/\gamma)_+-(1-1/\gamma)_+ + \delta\right].
\end{equation}
Though more refined bounds can be obtained, this relation shows that in some scenarios, DP mechanisms might provide conservative but loose lower bounds on the security level. \\
\noindent However, \textit{differential privacy is a tool to induce privacy into a model whereas the MIS a tool to measure the MIA-wise privacy of a learning procedure. }Specifically, the two frameworks are not equivalent and do not convey the same message. We further discuss it in Section \ref{sec:dp_vs_mis}.

\end{remark}

\noindent In Section \ref{overfitting}, we analyze how $\dnu$ is controlled when the learning procedure exhibits overfitting. In Section \ref{main_res_mean}, we address scenarios where we can provide precise control on $\dnu$. Numerical experiments are presented in Section \ref{exp}, and the proof of Theorem \ref{thm:dnu} is detailed in Appendix \ref{sec:proof:performances}. \\

%% file: Sections/v3/VII-Overfitting-v3.tex
In this section, we assume that $\gZ\coloneqq\gX\times\gY$ and that the learning procedure $\gA$ produces overfitting parameters $\hat{\theta}_n$. We then note $\rz\coloneqq(\rx, \ry)$.
We consider learning systems minimizing $L_n:\theta\mapsto \frac{1}{n}\sum_{j=1}^{n}l_{\theta}(\rx_j,\ry_j)$ for some training dataset $(\rz_1,\cdots,\rz_n)$ where $l_\theta : \gX\times\gY\to\R^+$ is a loss function. We defer all proofs of this section to Appendix \ref{sec:proof:overfitting}.


\begin{definition}[$(\varepsilon,1-\alpha)$-Overfitting]
\label{def:overfitting}
We say that the learning procedure $\gA$ is $(\varepsilon,1-\alpha)$-overfitting for some $\varepsilon\in\R^+$ and $\alpha\in(0,1)$ 
when
\begin{equation}
\label{eq:overfitting}
\mathbb{P}\left(l_{\hat{\theta}_n}(\rx_1, \ry_1) \leq \varepsilon\right) \geq 1-\alpha,
\end{equation}

\noindent where the probability is taken over the data and the randomness of $\gA$.

\end{definition}

\noindent When $\alpha=0$, \eqref{eq:overfitting} is equivalent to having $l_{\hat{\theta}_n}(\rx_j,\ry_j) \leq\varepsilon$ almost surely for all $j=1,\ldots,n$.
Furthermore, in many learning procedures, we give an additional stopping criteria taking the form $L_n\leq\eta$ for some $\eta\in\R^+$. Letting $\gA_\eta$ such a learning procedure, we give a sufficient condition for  \eqref{eq:overfitting} to hold:

\begin{proposition}
\label{prop:overfitting}
    For some fixed $\varepsilon\in\R^+$ and $\alpha\in(0,1)$, let $\eta\coloneqq\varepsilon\alpha$ and suppose that $\gA_\eta$ stops as soon as $L_n(\hat{\theta}_n)\leq\eta$. Then $\gA_\eta$ is $(\varepsilon,1-\alpha)$-overfitting.
\end{proposition}


\noindent We will need an additional hypothesis for the following theorem.

\noindent \textbf{Hypothesis (H1) :} $\gY\coloneqq \R^s$ for some $s\geq1$, and for all $\theta\in\Theta, x\in\gX$ and $y\in\gY$, we have
\begin{equation}
\label{eq:hyp}
    l_\theta(x,y) = \omega(y,\Psi_\theta(x)),
\end{equation}
for some family of functions $\Psi_\theta:\gX\to\R$ and some continuous function $\omega : \R\times\R \to \R$.


\begin{theorem}[Overfitting induces lack of security]
\label{thm:overfitting}
Assume $\gA$ is $(\varepsilon,1-\alpha)$-overfitting for some fixed $(\varepsilon,\alpha)$. 
Then we have
\begin{equation}
\label{eq:overfitting:multi}
\text{Sec}_n(P,\gA) \leq \max(1,1/\gamma)\left(\alpha + \gamma\mathbb{P}\left(l_{\hat{\theta}_n}(\rx,\ry)\leq\varepsilon\right)\right).
\end{equation}

\noindent Assume furthermore that \textbf{H1} holds. Assume that for all $\eta>0$, $\gA_\eta$ stops as soon as $L_n\leq\eta$ and that a version of the conditional distribution of $\ry$ given $\rx$ is absolutely continuous with respect to the Lebesgue measure, then 

\begin{equation}
\label{eq:overfitting:none}
\underset{\eta\to 0^+}{\lim}\text{Sec}_n(P,\gA)=0.
\end{equation}

\end{theorem}    

\noindent The second point of Theorem \ref{thm:overfitting} states that for regressors with reasonably low training loss on the dataset, a loss-based MIA $\phi_{\varepsilon}:(\theta,\rz)\mapsto 1_{l_\theta(\rz)\leq\varepsilon}$ would reach high success probability.
This theoretically confirms the already well-known insight that overfitting implies poor security.\\
Hypothesis \textbf{H1} occurs when $\Psi_\theta(x)$ models the conditional expectation of $\ry$ given $\rx$, in a setting where the loss function is defined as a distance between $\Psi_\theta(\rx)$ and $\ry$.
\begin{remark}
    Interestingly, if we only assume Definition \ref{def:overfitting} to hold without Proposition \ref{prop:overfitting} to hold, then a much weaker version of the second point of Theorem \ref{thm:overfitting} still holds. Indeed, for a fixed $\alpha\in(0,1)$, given a sequence of learning procedures $(\gA^{\varepsilon})_{\varepsilon\in\R^+}$ that are $(\varepsilon,1-\alpha)$-overfitting for all $\varepsilon>0$, we have that 
    $\underset{\varepsilon\to0}{\lim}\;\text{Sec}_n(P,\gA^{\varepsilon})\leq \max(1,1/\gamma)\alpha
    $.
\end{remark}
\begin{example}[non-linear regression Neural Network]
\label{ex:nonreglin}
We consider here a (non-linear) regression setting, that is for all $j=1,\ldots,n$, we have $y_j\coloneqq\Psi^*(\rx_j) + \zeta_j$, where $\zeta_j$ is some independent random noise and the function $\Psi^*:\gX\to\R$ is arbitrary, fixed and unknown.
We aim at estimating $\Psi^*$ by some Neural Network $\Psi_\theta\in\gF$, where $\gF$ is some fixed model. For instance $\gF$ can be the set of all 2-layers ReLU neural networks with fixed hidden layer width. The learning procedure $\gA$ then learns by minimizing the MSE loss $L_n\coloneqq\frac{1}{n}\sum_{j=1}^{n}\left(y_j - \Psi_\theta(x_j)\right)^2$. In this case, \eqref{eq:hyp} holds. Under the further assumption that there is an arbitrarily close approximation $\Psi_\theta$ of $\Psi^*$ in $\gF$, one can construct the sequence of learning procedures $(\gA_\eta)_{\eta\in\R^+}$ such that the hypotheses of the second point of Theorem \ref{thm:overfitting} for \eqref{eq:overfitting:none} to hold. Refer to Section \ref{exp} for a numerical illustration.
\end{example}

\begin{example}[Linear regression]
\label{ex:linreg:overfitting}
We assume here a linear regression setting, that is $\gX\coloneqq\R^d$ for some $d\in\mathbb{N}$, and $y_j\coloneqq \beta^T\rx_j + \zeta_j$, where $\zeta_j$ is some independent random noise and $\beta\in\R^d$ is fixed and unknown. Further assuming that $\zeta_j$ is absolutely continuous with respect to the Lebesgue measure, and that $d>n$, both \eqref{eq:overfitting} (with $\varepsilon,\alpha=0$) and \eqref{eq:hyp} hold. Then, the assumptions of the second point of Theorem \ref{thm:overfitting} are satisfied, leading to  $\text{Sec}_n(P,\gA)=0$.


\end{example}

%% file: Sections/v3/IV-V-Empiricalmeananddiscretecase-v3.tex
In this section, we study the converse, where we aim at understanding when to expect $\mis$ to be close to $1$.
All the proofs of the section can be found in Appendix \ref{sec:proof:means-discrete}.

\subsection{Empirical Mean based Learning Procedures}
\label{subsec:emp-mean}

We first study the case of learning procedures for which the parameters $\hat{\theta}_n$ can be expressed in the form of functions of empirical means (e.g., linear regression with mean-squared error, method of moments...). Specifically, for any (fixed) measurable maps $L:\gZ\to\R^d$ and $F :\R^d\to\R^q$ for some $d,q\in\mathbb{N}$, 
we consider that 
\begin{equation}
\label{eq:funcmean}
{\hat{\theta}_n\coloneqq F\left(\frac{1}{n}\sum_{j=1}^{n}L(z_j)\right)}.
\end{equation}

\noindent Equation \ref{eq:funcmean} states that the parameters are the result of the learning procedure $\gA : (z_1,\cdots,z_n)\mapsto \delta_{F\left(\frac{1}{n}\sum_{j=1}^{n}L(z_j)\right)}$, where $\delta_\theta$ stands for the Dirac mass at $\theta$. We then have the following result.

\begin{theorem}
\label{theo:empmean}
Suppose that the distribution of $L(\rz_1)$ has a non zero absolutely continuous part with respect to the Lebesgue measure, and a third finite moment.
Then
\begin{equation}
\label{eq:empmean}
\mis\geq 1 - \max(1,\gamma)\left(c_{L,P}  + \frac{\sqrt{d}}{2\sqrt{n}}\right)n^{-1/2},
\end{equation}
\noindent for some constant $c_{L,P}$ depending only on $L$ and $P$.
\end{theorem}
\noindent \textbf{Remark \ref{theo:empmean}} : \label{eq:omega} Theorem \ref{theo:empmean} implies that a sufficient condition to ensure $Sec_n(P,\gA)$ to be made larger than $1-\max(1,\gamma)\varepsilon$, is to have  
$n\geq\Omega(\varepsilon^{-2})$. The hidden constant only depends on the distribution data $P$ and the parameters dimension $d$. See Appendix \ref{proof:mean:minn} for a proof.

\noindent We now provide examples for which Theorem \ref{theo:empmean} allows us to give an lower bound on $\mis$.

\begin{example}[solving equations]
We seek to estimate an (unknown) parameter of interest $\theta_0\in\Theta\subseteq\R^d$. We suppose that we are given two functions $h:\Theta\to\R^l$ and $\psi:\gZ\to\R^l$ for some $l\in\mathbb{N}$, and that $\theta_0$ is solution to the equation 
\begin{equation}
\label{eq:gmm}
h(\theta_0) = \E[\psi(\rz)],
\end{equation}
where $\rz$ is a random variable of distribution $P$. Having access to data samples $\rz_1,\ldots,\rz_n$ drawn independently from the distribution $P$, we estimate $\E[\psi(\rz)]$ by $\frac{1}{n}\sum_{j=1}^{n}\psi(\rz_j)$. Assuming that $h$ is invertible, one can set $\hat{\theta}_n = h^{-1}\left(\frac{1}{n}\sum_{j=1}^{n}\psi(\rz_j)\right)$, provided that $\frac{1}{n}\sum_{j=1}^{n}\psi(\rz_j)\in \mathcal{I}m(h)$.
In particular, when $\gZ=\R$, this method generalizes the method of moments by setting $\psi(z) = (z,z^2,\cdots,z^l)$. We then may apply Theorem \ref{theo:empmean} to any estimators obtained by solving equations.
\end{example}

\begin{example}[Linear Regression]
\label{ex:linreg}
We consider here the same framework as in Example \ref{ex:linreg:overfitting}, where $d<n$ (hence Definition \ref{def:overfitting} can not be fulfilled with $\alpha=0$). 
Let $\sX$ be the $d\times n$ matrix whose $i^{th}$ row is $\rx_i$, and $\sY$ be the column vector $(\ry_1,\cdots,\ry_n)^T$. We then recall that the estimator $\hat{\beta}_n$ of $\beta$ is given by
$${\hat{\beta}_n \coloneqq (\sX\sX^T)^{-1}\sX\sY^T}.$$
Based on \eqref{eq:funcmean}, if we set $F(K,b) \coloneqq K^{-1}b^{T}$ and $L((x,y)) \coloneqq ((x^ix^j)_{i,j=1}^d,(x^iy)_{i=1}^d)$, where $x^i$ is the $i^{\text{th}}$ coordinate of $x$, then the estimator can be expressed as follows 
$\hat{\beta}_n = F\left(\frac{1}{n}\sum_{j=1}^{n}L((\rx_j,\ry_j))\right).$ 
\end{example}
Interestingly, we see from Examples \ref{ex:linreg:overfitting} and \ref{ex:linreg} that the security of least squares linear regression estimator is constant $0$ up to $n=d$ (where $d$ is both the dimension of the data and the dimension of the parameters), and then is increasing up to $1$ provided that $n\rightarrow \infty$.

\subsection{Discrete Data Distribution}

\label{subsec:disc}

We now consider the common distribution of the points in the data set to be $P\coloneqq\sum_{k=1}^{K}p_k\delta_{u_k}$ for some fixed $K\in\mathbb{N}\cup\{\infty\}$, some fixed probability vector  $(p_1,\cdots,p_K)$ and some fixed points $u_1,\ldots,u_K$ in $\gZ$.
Without loss of generality, we may assume that $p_k>0$ for all $k\in\{1,\cdots,K\}$. In the discrete distribution setting, we show that the convergence of the MIS toward 1 can occur at different rates (see Appendix \ref{app:diff_rates}), depending on both the algorithm and the underlying data distribution. Therefore, in what follows, we are interested in studying the worst security among all learning procedures, namely ${\min}_\gA\;\mis$.
\begin{theorem}
\label{theo:disc}
For $k=1,\ldots,K$, let $B_k$ be a random variable having Binomial distribution with parameters $(n,p_k)$. Then,
\begin{equation}
\label{eq:minsecdisc}
\underset{\gA}{\min}\;\mis = 1 - \max(1,\gamma)\frac{1}{2\gamma} \sum_{k=1}^{K}\E\left[\left|\frac{B_k}{n}- \gamma p_k\right|\right] + \max(1,\gamma)\frac{1}{2}\left|1-\finv\gamma\right|,
\end{equation}
\noindent where the minimum is taken over all learning procedures and is reached on learning procedures of the form $\gA(z_1,\cdots,z_n) = \delta_{F(\frac{1}{n}\sum_{j=1}^{n}\delta_{z_j})}$ for some injective maps $F$.

\end{theorem}

\noindent Theorem \ref{theo:disc} provides an exact formula to accurately bound $\mis$ for any learning procedure $\gA$, including those that exhibit the most significant leakage. It is shown in Theorem \ref{theo:disc} that the minimum is reached for deterministic learning procedures $\gA$. We show below that the r.h.s. of \eqref{eq:minsecdisc} is tightly related to the 
quantity
\begin{equation}
C_K(P)\coloneqq \sum_{k=1}^K\sqrt{p_k(1-p_k)}.
\end{equation}

\noindent This quantity can thus be exploited to compare leakage between different datasets.

\noindent It is worth noting that $C_K(P)$ is a diversity measure, giving a control on the homogeneity of the data distribution. We show in Appendix \ref{sec:cp} that it is comparable both to the Gini-Simpson and the Shannon Entropy. \\
\noindent Unlike $\lambda$ begin a meta-parameter, the theoretical value of $\nu$ is never known in practice. In the following corollary, we exhibit the security of the worst learning procedure privacy-wise in the worst theoretical setup possible.

\begin{corollary}
\label{cor:cp}
Assume that $C_K(P)<\infty$, $n\geq 5$ and $n>1/p_k$ for all $k=1,\ldots,K$.  Then there exists universal constants $c\geq0.29$ and $c'\leq0.44$ such that 
$$\underset{\nu, \gA}{\min}\;\mis = 1 - \varepsilon_n,$$
\noindent where $\varepsilon_n$ satisfies
$$cC_K(P)n^{-1/2} \leq \varepsilon_n \leq c'C_K(P)n^{-1/2}.$$
\noindent If $C_K(P)<\infty$ but the condition on $n$ does not hold, we still have $\underset{\nu, \gA}{\min} \mis\geq 1 -\frac{C_K(P)}{2}n^{-1/2} $.

\end{corollary}
Corollary \ref{cor:cp} implies that no matter the theoretical value of $\nu$, a sufficient condition to ensure security larger than $1-\varepsilon$ is to have at least $n\geq(C_K(P)/2\varepsilon)^2$. \\
Importantly, when designing a learning procedure, one can never have access to the true theoretical value $\nu$. Consequently, Corollary \ref{cor:cp} provides a practical way of measuring the security, without requiring unavailable knowledge.\\
\noindent In any case, this result indicates that discretizing the data allows the control on the security of any learning procedure, no matter the specified value of $\lambda$.  The condition $n\geq 1/p_k$ shall be understood as the expected number of samples to observe the atom $u_k$ within a dataset. Indeed, if $\Xi_k$ is the random variable representing the required number of samples to observe the atom $u_k$, then $\Xi_k$ follows a Geometric law with parameter $p_k$, hence its expectation is $\mathbb{E}[\Xi_k] = 1/p_k$. We further discuss it in Section \ref{sec:cp}. Section \ref{exp} illustrates the impact of $C_K(P)$ through numerical experiments. We further discuss this Corollary in Section \ref{sec:cp}.

%% file: Sections/v3/VIII-numericalexp-v3.tex
In this section, we propose two  numerical experiments to illustrate our results in Sections \ref{overfitting} and \ref{main_res_disc}. All simulations have been conducted with the Pytorch library. We refer to Appendix \ref{sec:exp} for more details on 
the experiments.

\subsection{Overfitting}
\label{exp:overfitting}

We run three experiments to illustrate to illustrate the results of Section \ref{overfitting}.\\
\noindent \textbf{Regression task.} We run a non-linear regression experiment to illustrate the results of Section \ref{overfitting} and specifically Example \ref{ex:nonreglin}. We then consider the setting of Example \ref{ex:nonreglin} with $\Psi^*(x) = \sin(\pi\beta^T x)$ for some fixed $\beta$. During the training of the neural network $\Psi_{\hat{\theta}_n}$, at each iteration 
we evaluate the fraction of training (validation) data that achieves a loss below $\varepsilon$. Validation data correspond to a set of data independent from the training dataset.

\begin{figure}[ht]
    \includegraphics[width=\linewidth]{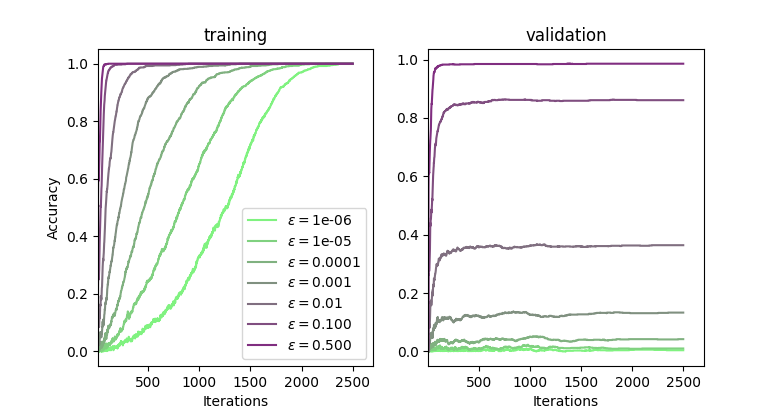}
    \caption{Shows the fraction of the Training/Validation dataset whose loss is under given thresholds during the training process. The left figure shows the \textbf{training accuracy}, and the right figure shows the \textbf{validation accuracy}.}
    \label{fig:overfitting}
\end{figure}

\begin{figure}
\centering
\begin{subfigure}{.5\textwidth}
  \centering
  \includegraphics[width=\linewidth]{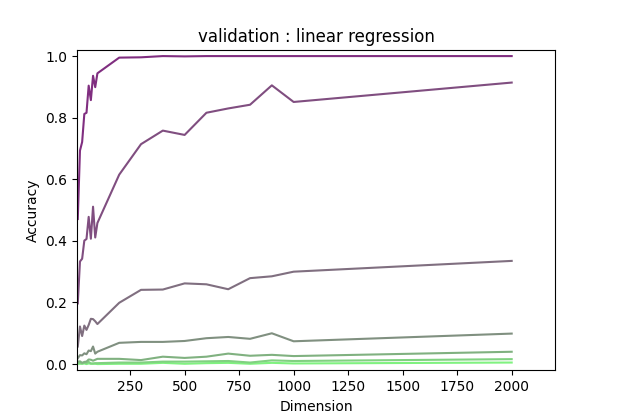}
  \caption{Linear regression.}
  \label{fig:overfitting-RegLin}
\end{subfigure}%
\begin{subfigure}{.5\textwidth}
  \centering
  \includegraphics[width=\linewidth]{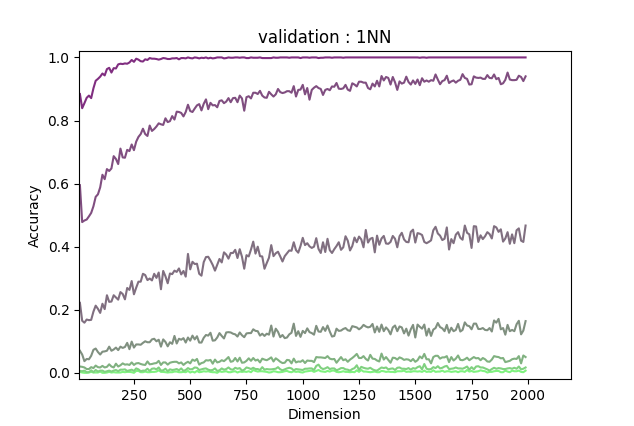}
  \caption{Nearest Neighbors (k=1)}
  \label{fig:overfitting-KNN}
\end{subfigure}
\caption{Shows the fraction of the Validation dataset whose loss is under given thresholds at the end of the training process for different dimensions. The left figure shows the \textbf{validation accuracy} for the linear regression model , and the right figure shows the \textbf{validation accuracy} for the nearest neighbors model.}
\label{fig:test}
\end{figure}



\noindent Figure \ref{fig:overfitting} illustrates Theorem \ref{thm:overfitting} by showing that for very small values of threshold ($\varepsilon = 10^{-6}$), we still reach $100\%$ training accuracy after $2500$ iterations whereas the validation accuracy (for $\varepsilon = 10^{-6}$) stabilizes at near $0\%$. \\

\noindent To illustrate Example \ref{ex:linreg:overfitting}, we train multiple linear regressors until complete overfitting (training loss near $0$) and evaluate it at the end of its training. We perform this evaluation for various input dimensions from $10$ to $2000$. For each dimension, we evaluate the linear regressor and report the fraction of the validation dataset whose loss in under given thresholds. Figure \ref{fig:overfitting-RegLin} illustrates Theorem \ref{thm:overfitting} by showing that for very small values of threshold ($\varepsilon=10^{-6}$), the validation accuracy remains near $0\%$ for any dimension.\\
\noindent Therefore, simple loss-based MIA with threshold $\varepsilon$ would suffice to accurately predict membership most of the time both for the non-linear regression and the linear regression tasks. The number of iterations being generally unknown to the MIA, the calibration of $\varepsilon$ is a hard task to perform. Even though it seems that the loss-based attack with threshold $\varepsilon=10^{-6}$ is a good candidate to achieve near perfect guess, it is worth noting that it would occur only if at least 2000 iterations (for the non-linear regression task) have been done during the training procedure.\\
\noindent \textbf{Nearest Neighbors.} We additionally illustrate overfitting through a $k-$Nearest Neighbors model, with $k=1$. Similarly to the linear regression task, we perform the evaluation for various input dimensions from $10$ to $2000$. By construction of the Nearest Neighbors algorithm, the error main by the predictor on the training dataset is always $0$. Figure \ref{fig:overfitting-KNN} shows a similar behavior as for the regression tasks, meaning that simple loss-based MIAs would also suffice to to accurately predict membership.

\subsection{Impact of $C_K(P)$ on accuracy}
\label{exp:cp}
Corollary \ref{cor:cp} indicates that discretizing the data distribution improves the security of the model. To illustrate the impact of a discretization through the constant $C_K(P)$, we trained several 3-layers neural networks to classify samples from the MNIST dataset \citep{deng2012mnist}.
Before training, we fixed three discretizations\footnote{Many clustering algorithms exist, but we did not aim at optimizing the choice of the discretizations.} (clusterings). For each dataset (with varying size $n$), we trained a neural network on it, and three other neural networks on discretized versions of the original dataset (one for each clustering). We then numerically computed the quantity $C_K(P)$ for each discretization.
 \begin{table}[htbp]
    \centering
    \resizebox{1\textwidth}{!}{
        \begin{tabular}{c c c c c}
        \toprule
        n & raw dataset & $C_K(P)=4.3$ & $C_K(P)=6.74$ & $C_K(P)=9.20$ \\
        \midrule
        $1000$ & $0.989 \pm 0.0011 $ & $0.963 \pm 0.0223$  $(\Delta_{\nu,\lambda,n} \leq 0.07)$ & $0.968 \pm 0.0137$  $(\Delta_{\nu,\lambda,n} \leq 0.11)$ & $0.986 \pm 0.0039$  $(\Delta_{\nu,\lambda,n} \leq 0.15)$ \\
             $5000$ & $0.993 \pm 0.0012 $ & $0.967 \pm 0.0184$  $(\Delta_{\nu,\lambda,n} \leq 0.03)$ & $0.971 \pm 0.0282$  $(\Delta_{\nu,\lambda,n} \leq 0.05)$ & $0.984 \pm 0.0055$  $(\Delta_{\nu,\lambda,n} \leq 0.07)$ \\ 
                       $10000$ & $0.994 \pm 0.0006 $ & $0.971 \pm 0.0141$  $(\Delta_{\nu,\lambda,n} \leq 0.02)$ & $0.977 \pm 0.0082$  $(\Delta_{\nu,\lambda,n} \leq 0.03)$ & $0.984 \pm 0.0055$  $(\Delta_{\nu,\lambda,n} \leq 0.05)$ \\ 
        \bottomrule
        \\

        \end{tabular}
    }
    \caption{Shows the accuracy of classifiers on MNIST dataset. The column \textbf{n} displays the dataset size. The column \textbf{raw dataset} displays the accuracy of the neural network on the original dataset. Each column of the column \textbf{discretized datasets} displays the accuracy of a neural network on the discretized dataset associated to the constant $C_K(P)$.}
    \label{fig:cp}
\end{table}
Table~\ref{fig:cp} shows the accuracy of all neural networks, and displays the impact of the discretization on the accuracy, depending on $n$ and the value of $C_K(P)$. For a dataset of size $n=1000$, our neural network reaches an accuracy of $0.989$ when trained on the original dataset. When discretizing, Table \ref{fig:cp} displays a loss of almost $2.5\%$ of accuracy for the discretization having $C_K(P) = 4.30$, and a loss of $2\%$ for the other discretizations. As discussed in Section \ref{sec:cp}, increasing the number of clusters will increase the value of $C_K(P)$. Table~\ref{fig:cp} displays the intuition that smaller discretization (smaller $C_K(P)$) will lower simultaneously the accuracy and the quantity $\dnu$, which motivates the need to optimize the discretization to find a trade-off between security and accuracy.

%% file: Sections/v3/VI-commentsandremarks-v3.tex
The findings presented in this article open gates to the theoretical understanding of MIAs, and partially confirm some of empirically observed facts. Specifically, we confirmed that overfitting indeed induces the possibility of highly successful attacks. We further revealed a sufficient condition on the size of the training dataset to ensure control on the security of the learning procedure, when dealing with discrete data distributions or functionals of empirical means. We established that the rates of convergence consistently follow an order of $n^{-1/2}$. The constants established in the rates of convergence scale with the number of discrete data points and the dimension of the parameters in the case of functionals of empirical means. Interestingly, for discrete data, the diversity measure $C_K(P)$ highlights the use of data quantization to ensure privacy by design.\\
\noindent
The security measure $\mis$ can be estimated by approximating the optimal accuracy-maximizing attack with a binary classifier (e.g., a feed-forward neural network) and evaluating its classification accuracy. The optimal attack is theoretically known and is given by $\phi^*(\theta,\rz) = 1_{(\theta,\rz)\in B^*}$, where $B^*$ denotes the set reaching the maximum in the definition. However, determining this set explicitly depends on the densities of the joint and product distributions between $\hat{\theta}_n$ and $\rz_1$. Direct evaluation of these densities is computationally demanding, highlighting the necessity of practical estimation approaches. We therefore believe that a thorough investigation into the estimation of $\mis$ constitutes a substantial research effort that requires a dedicated full-length study.




\noindent In Section \ref{main_res_disc}, our work is currently limited to the discrete data (e.g., tabular datasets) and empirical mean based learning procedures. We intend to extend further our research to the complete study of maximum likelihood estimation, empirical loss minimization and Stochastic Gradient Descent. Moreover, we studied the worst possible learning procedure by taking the supremum over all learning procedures. An interesting question would be to analyze the security of specific sets of learning procedures. We leave this to further works.

\noindent Furthermore, we aim to explore the optimization of the trade-off discussed at the conclusion of Section \ref{exp:cp}.

\noindent Our findings about overfitting learning procedures do not 
cover classification learning procedures. We anticipate continuing our research in this direction to gain a comprehensive understanding of the impact of overfitting. 
Currently, we are able to extend these results to a minor scenario (see Appendix \ref{sec:overfitting} Proposition \ref{prop:classif}). We anticipate that these assumptions may be relaxed in future investigations. \\
\noindent We have also assumed the data to be \textit{i.i.d.}. Often, concentration inequalities and asymptotic theorems over \textit{i.i.d.} samples have their counterparts in the dependent setting. Assuming some dependence in the data is therefore an interesting way to extend our results.

%% file: Appendix/v3/G-CommentsonSetup-v3.tex
We further discuss here the assumption that the learning procedure shall be expressed as a function of the empirical distribution of the training dataset. \\
Usual definition of a learning procedure $\gA$ asks for its domain to be $\bigcup_{n\geq1}\gZ^n$ which is similar to identifying it with a sequence of learning procedures $\left(\gA_n : \gZ^n\to\gP'\right)_{n\geq1}$ where for each $n\geq1$ we have that the restriction of $\gA$ to $\gZ^n$ coincides with $\gA_n$. However, this definition allows not specifying its behavior through all values of $n$, and specifically having drastically different behaviors for different values of $n$. To rigorously study the characteristics of a learning procedure, it is natural to ask that its behavior is similar for all values of $n$, meaning that its behavior can be defined independently of $n$.\\
Our assumption solves this issue, as the function $G:\gM\to\gP'$ from Section \ref{prob_form} is defined for all discrete distribution on $\gZ$. Furthermore, it is worth noting that this assumption holds for all learning procedures aiming at minimizing the empirical cost on its training dataset. For most learning procedures, it will still hold even when some weights are applied to the samples. Indeed, changing the distribution $P$ by $P'$ for some other distribution $P'$, the training dataset size $n$ by some other integer $n'\in\mathbb{N}$ and adequately adapt the learning procedure $\gA$ into another learning procedure $\gA'$ to take into account the changes,  makes the study still valid as long as we consider $\Delta_{\nu,\lambda,n'}(P',\gA')$ instead of $\dnu$. Therefore, it is sufficient to conduct the study under this hypothesis.\\
In particular, this assumption treats all points similarly, and is invariant with respect to the redundancy of the whole dataset. It is equivalent to saying that the learning procedure is \textbf{symmetric} and \textbf{redundancy invariant}, whose definitions are given below.

\begin{definition}[Symmetric Map]
\label{def:symm}
Given two sets $\gZ_1$ and $\gZ_2$ and an integer $k$, a map $f:\gZ_1^k\to \gZ_2$ is said to be \textbf{symmetric} if for any $(a_1,\cdots,a_k)\in \gZ_1^k$ and any permutation $\sigma$ on $\{1,\cdots,k\}$, we have 
$$
f(a_1,\cdots,a_k) = f\left(a_{\sigma(1)},\cdots,a_{\sigma(k)}\right).
$$
\end{definition}

\begin{definition}[Redundancy Invariant Map]
\label{def:redun}
Given two sets $\gZ_1$ and $\gZ_2$, a map $f:\bigcup_{k>0}\gZ_1^k\to \gZ_2$ is said to be \textbf{redundancy invariant} if for any integer $m$ and any $\va=(a_1,\cdots,a_m)\in \gZ_1^m$, we have 

$$f(\va) = f(\va,\cdots,\va).$$

\end{definition}

\noindent We summarize the last claim in the following proposition.

\begin{proposition} 
\label{prop:G}
Let $f : \bigcup_{k>0}\gZ^k\to\gZ'$ be a measurable map onto any space $\gZ'$. Then the following statements are equivalent

\begin{itemize}
    \item[(i)] $f$ is redundancy invariant and for any $k\in\mathbb{N}$, the restriction of $f$ to $\gZ^k$ is symmetric. 
    \item[(ii)] There exists a function $G:\gM\to\gZ'$ such that for any $k\in\mathbb{N}$, for any $(z_1,\cdots,z_k)\in\gZ^k$ we have $f(z_1,\cdots,z_k) = G\left(\frac{1}{k}\sum_{j=1}^{k}\delta_{z_j}\right)$. 
\end{itemize}
\end{proposition}

\begin{proof}[Proof of Proposition \ref{prop:G}]
We only prove that $(i)$ implies $(ii)$. The fact that $(ii)$ implies $(i)$ is straightforward. \\ 
Let $f : \bigcup_{k>0}\gZ^k\to\gZ'$ be a measurable map satisfying condition $(i)$. Let $\gM^{\text{emp}}$ be the set of all possible empirical distributions, that is the subset of $\gM$ containing all $\frac{1}{k}\sum_{j=1}^{k}\delta_{z_j}$ for all integer $k$ and all $(z_{1},\cdots,z_{k})\in \gZ^{k}$. We shall define $G$ on $\gM^{\text{emp}}$ such that $(ii)$ holds true.\\
For any $Q\in \gM^{\text{emp}}$, let $\{z_{1},\cdots,z_{m}\}$ be its support and $q_{1},\ldots,q_{m}\in(0,1)$ be such that $Q=\sum_{j=1}^{m}q_{j}\delta_{z_j}$.   Since $Q$ is an empirical distribution, there exists positive integers $k_{1},\ldots,k_{m}$ (for each $j$, $k_{j}$ is the number of occurences of $z_{j}$ in the sample from which $Q$ is the empirical distribution) such that $q_{j}=\frac{k_{j}}{K}$, with $K=\sum_{j=1}^{m}k_{j}$.\\
Let $r=gcd(k_{1},\ldots,k_{m})$ be the greatest common divisor of the $k_j$'s and define $k'_{j}=k_{j}/r$ for $j=1,\ldots,m$. Then with $K'\coloneqq\sum_{j=1}^{m}k'_{j}$, we have $q_{j}=\frac{k'_{j}}{K'}$.\\
Now, for any other sequence of positive integers $\ell_{1},\ldots,\ell_{m}$ such that $q_{j}=\frac{\ell_{j}}{L}$, with $L=\sum_{j=1}^{m}\ell_{j}$, we get for all $j$, $\ell_{j}=s k'_{j}$ with $s=gcd(\ell_{1},\ldots,\ell_{m})$. Thus we may define 
$G(Q)=f (\vz)$ where $\vz$ is the dataset consisting of all $z_{j}$'s with $k'_{j}$ repetitions.\\ 
We now prove that such a $G$ satisfies $(ii)$. Indeed, for any integer $k$ and any
$Z \coloneqq (z'_1,\cdots,z'_k)\in\gZ^k$, define  $V\coloneqq((\ell_1,z_1),\cdots,(\ell_m,z_m))$
where $(z_1,\cdots,z_m)$ are the distinct  elements of $Z$ and $(\ell_1,\cdots,\ell_m)$ are their occurrences. Define $r$ as their greatest common divisor, and $(k_{1},\ldots,k_{m})=(\ell_1,\cdots,\ell_m)/r$. By using the fact that $f$ is symmetric and redundancy invariant, we get that $f(Z)=f(Z_{0})=G(Q)$ where $Z_{0}$ is the dataset  consisting of all $z_{j}$'s with $k_{j}$ repetitions and $Q=\sum_{j=1}^{m}\frac{k_{j}}{K}\delta_{z_j}=\frac{1}{n}\sum_{j=1}^{n}\delta_{z'_j}$. Thus $(ii)$
 holds.    
\end{proof}

%% file: Appendix/v3/I-CommentonOverfitting-v3.tex
In this section, we give an extension of Theorem \ref{thm:overfitting} to the setting of classification.

\noindent The second point of Theorem \ref{thm:overfitting} requires the absolute continuity of the distribution of the label with respect to the Lebesgue measure, which makes it not straightforward to extend it to classifiers.\\
We discuss here one very specific framework in which we have been able to extend our results to the classification setting. The framework and the assumptions are all inspired from \citet{vardi2022gradient}.\\

\noindent We assume that the data space is restrained to the binary classification setting with data in the sphere of radius $\sqrt{s}$, i.e. $\gZ\coloneqq\left(\sqrt{s}\sS^{s-1}\right)\times\{-1,1\}$ where $\sS^{s-1}$ is the unit sphere in $\R^{s}$. We assume our data $(\rz_1,\cdots,\rz_n)\coloneqq((\rx_1,\ry_1),\cdots,(\rx_n,\ry_n))$ to be independently drawn on $\gZ$ from a distribution $P$. We assume that the conditional law of $\rx_1$ given $\ry_1$ is absolutely continuous with respect to the Lebesgue measure on $\sqrt{s}\sS^{s-1}$. We denote by $\gH$ the latter hypothesis.
Let $\Psi_\theta(x) = \sum_{j=1}^{l}v_j\sigma(w_j^T x + b_j)$ be a $2-$ReLU network with parameters $\theta$, i.e. $\theta = (v_j, w_j, b_j)_{j=1}^{l}$ with $l\in\sN$ the width of the network and $\sigma(u) = \max{(u,0)}$. We aim at learning a classifier $\Psi_{\hat{\theta}_n}$ on the data by minimizing 

\begin{equation}
\label{eq:min:classif}
\gL : \theta\mapsto \sum_{j=1}^{n}l(y_j\Psi_\theta(x_j)),
\end{equation}

\noindent where $l:\R\to\R^+$ is either the exponential loss or the logistic loss. To reach the objective, we apply Gradient Flow on the objective \eqref{eq:min:classif}, producing a trajectory $\theta_n(t)$ at time $t$. From \citet{vardi2022gradient} Theorem 3.1, there exists a $2-$ReLU network classifying perfectly the training dataset, as long as $\underset{i\not=j}{\max}\left\{|\rx_i^T\rx_j|\right\}<d$, which holds almost surely by $\gH$. Let the initial point $\theta_n(0)$ be the parameters of this network.

\noindent Then by \citet{vardi2022gradient} Theorem 2.1, paraphrasing \citet{lyu2019gradient, ji2020directional}, $\frac{\theta_n(t)}{\|\theta_n(t)\|}$ converges as $t$ tends to infinity to some vector $\Bar{\theta}_n$ which is colinear to some KKT point of the following problem 

\begin{equation}
\label{eq:kkt}
\underset{\theta}{\min}\frac{1}{2}\|\theta\|^2\;\text{ s.t. }\;\forall i=1,\ldots,n;y_i\Psi_\theta(x_i)\geq1.
\end{equation}

\noindent Conditional to the event $E\coloneqq"\underset{i\not=j}{\max}\left\{|\rx_i^T\rx_j|\right\}\leq\frac{s+1}{3n}-1"$, by \citet{vardi2022gradient} Lemma C.1 we get that for all $j=1,\ldots,n$, we have

\begin{equation}
\label{eq:classif:maxmargin}
\ry_j\Psi_{\Bar{\theta}_n}(\rx_j) = \lambda(\rz_1,\cdots,\rz_n),
\end{equation}

\noindent for some $\lambda(\rz_1,\cdots,\rz_n)>0$.

\noindent We consider our learning procedure $\gA$ to output $$\gA(\rz_1,\cdots,\rz_n) = \hat{\theta}_n\coloneqq\frac{\Bar{\theta}_n}{\sqrt{\lambda(\rz_1,\cdots,\rz_n)}},$$

\noindent which gives the same classifier as with $\Bar{\theta}_n$.

\noindent We then get the following result.

\begin{proposition}
\label{prop:classif}
Assume that $l\geq n$ and let $C\coloneqq \underset{i\not=j}{\max}\left\{|\rx_i^T\rx_j|\right\}$. Then, there exists an initialization $\theta_n(0)$ of the gradient flow for which it holds that 

$$\dt \geq \flg \mathbb{P}\left(C \leq \frac{s+1}{3n}-1\right) + \left(1-\flg\right).$$
\noindent Moreover, if the marginal distribution of $\rx$ is the uniform distribution on $\sqrt{s}\sS^{s-1}$, then 
$$ \mathbb{P}\left(C \leq \frac{s+1}{3n}-1\right) \geq 1 - s^{3-ln(s)/4},$$
\noindent as soon as $n\leq\frac{1}{3} \frac{s+1}{\sqrt{s}ln(s)+1}.$

\end{proposition}

\noindent Proposition \ref{prop:classif} with Theorem \ref{thm:dnu} show that the MIS is then upper bounded as
\begin{align*}
\mis &\leq \max(1,\lg)s^{3-ln(s)/4}.  
\end{align*}
\\

\begin{proof}[Proof of Proposition \ref{prop:classif}]
 By definition of $\Psi_\theta$ for any $\theta\in\Theta$, it holds that these networks are $2-$homogeneous, so that conditional to the event $E$, \eqref{eq:classif:maxmargin} leads to 

\begin{equation}
\label{eq:eq1}
\ry_j\Psi_{\hat{\theta}_n}(\rx_j) = 1,
\end{equation}

\noindent for any $j=1,\ldots,n$.

\noindent Let $S\coloneqq\{(\theta,x,y)\in\Theta\times\left(\sqrt{s}\sS^{s-1}\right)\times\{-1,1\} : y\Psi_\theta(x) = 1\}$. Then, by definition of $\Delta_n(P,\gA)$, we have 

\begin{align*}
\dt &\geq \flg\mathbb{P}((\hat{\theta}_n,\rx_1,\ry_1)\in S) -  \mathbb  {P}((\hat{\theta}_n,\rx,\ry)\in S) + \left(1-\flg\right) \\
    &=\flg \mathbb{P}((\hat{\theta}_n,\rx_1,\ry_1)\in S\mid E)\mathbb{P}(E) + \flg \mathbb{P}((\hat{\theta}_n,\rx_1,\ry_1)\in S\mid E^c)\mathbb{P}(E^c) + \left(1-\flg\right) \\&- \E\left[\mathbb{P}(\Psi_{\hat{\theta}_n}(\rx) = \ry\mid \hat{\theta}_n,\ry)\right] \\ 
    &\geq \flg \mathbb{P}((\hat{\theta}_n,\rx_1,\ry_1)\in S\mid E)\mathbb{P}(E) + \left(1-\flg\right)-\E\left[\mathbb{P}(\Psi_{\hat{\theta}_n}(\rx) = \ry\mid \hat{\theta}_n,\ry)\right],
\end{align*}

\noindent where we have lower bounded the second term by 0.\\
By \eqref{eq:eq1}, we have $\mathbb{P}((\hat{\theta}_n,\rx_1,\ry_1)\in S\mid E) = 1$. Now, by independence between $(\rx,\ry)$ and $\hat{\theta}_n$, it is sufficient to show that for any $\theta\in\Theta$, we have $P(\Psi_\theta(\rx) = \ry \mid \ry) = 0$ almost surely. Without loss of generality, we may assume that $v_j\not=0$ for any $j=1,\ldots,l$. We set $B_J(\rx,\ry)\coloneqq\left\{\forall j\in J, w_j^T\rx + b_j > 0\right\}\cap \left\{\forall j \in J^c, w_j^T\rx + b_j\leq0\right\}\cap\left\{\sum_{j\in J}v_j\left(w_j^T\rx+b_j\right)=\ry\right\}$ for any $J\subseteq[1,\cdots,l]$. We then get 

\begin{align*}
    \mathbb{P}\left(\Psi_{\hat{\theta}_n}(\rx) = \ry\mid\ry\right) &= \sum_{J\subseteq[1,\cdots,l]}P\left(B_J(\rx,\ry)\mid \ry\right) \\ 
    &\leq \sum_{J\subseteq[1,\cdots,l]}P\left(\sum_{j\in J}v_j\left(w_j^T\rx+b_j\right)=\ry\mid \ry\right).
\end{align*}

\noindent Note that the space $H_{y,J}\coloneqq\left\{x\in\R^s : \sum_{j\in J}v_j\left(w_j^T\rx+b_j\right)=\ry\right\}$ is an hyperplan of $\R^s$ for any $y\in\{-1,1\}$ and any $J\subseteq[1,\cdots,l]$. Then the quantity $P(\rx\in H_{y,J}\mid y)$ equals $0$ by $\gH$. Hence,

\begin{align*}
    \tilde D_\gamma(\mathbb{P}_{(\hat{\theta}_n,\rz_0)},\mathbb{P}_{(\hat{\theta}_n,\rz_1)}) &\geq \flg \mathbb{P}(E) + \left(1-\flg\right).
\end{align*}

\noindent Under the further hypothesis that $\rx$ is uniformly distributed on the sphere, and that $n\leq\frac{1}{3}\frac{s+1}{\sqrt{s}ln(s)+1}$, it holds that $\frac{s+1}{3n}-1\geq \frac{\sqrt{s}}{ln(s)}$. Then \citet{vardi2022gradient} Lemma 3.1 concludes.

\end{proof}

%% file: Appendix/v3/F-MorecommentsonCP-v3.tex

We give here some more details about the behavior of $\dnu$ when the set of parameters $\Theta$ has finite cardinal. We also further discuss the quantity $C_K(P)$.

\subsection{Different rates for $\dnu$}
\label{app:diff_rates}

Corollary \ref{cor:cp} gives a rate of $n^{-1/2}$ for $\underset{\nu,\gA}{\max}\;\dnu$ when $C_K(P)<\infty$ and $n$ is sufficiently large. In the case when $C_K(P)$ is infinite, it is interesting to note that we still have convergence to 1 of $\underset{\nu, \gA}{\min}\;\mis$ but at an arbitrarily slow rate. We formalize this result in the following lemma :

\begin{lemma}
\label{lem:infcp}
If $C_K(P)=\infty$, $\underset{\nu, \gA}{\min}\;\mis$ still tends to $1$ as $n$ tends to infinity, but the (depending on $P$) rate can be arbitrarily slow.
\end{lemma}

\noindent In this case, in order to find the minimum amount of data to get a control on $Sec_n(P,\gA)$ requires the estimation of the r.h.s. of \eqref{eq:minsecdisc} which is not obvious, as the condition $C_K(P)=\infty$ requires $K=\infty$.

\begin{proof}
Remarking that $\dt\leq\|\mathbb{P}_{(\hat{\theta}_n,\rz_0)}-\mathbb{P}_{(\hat{\theta}_n,\rz_1)}\|_{TV} + \left(1-\lg\right)$, it is a direct corollary of Lemmas 7 and 8 of \citet{binMAD2013}.
\end{proof}

\noindent Corollary \ref{cor:cp} also requires $n$ to be at least $1/p_j$ for any $j=1,\ldots,K$. Let $C_q(P) = \sum_{k=1}^{q}\sqrt{p_k(1-p_k)}$, where $p_1\geq\ldots\geq p_K$. Without this assumption, Corollary \ref{cor:cp} becomes

\begin{lemma}
\label{lem:cp-without-n}
Assume $C_K(P) < \infty$ and $n\geq 5$. Assume without loss of generality that ${p_1\geq\cdots\geq p_K}$, and let $k_n \coloneqq max\{k : p_k \geq 1/n\}$. Then we have
$$\underset{\nu, \gA}{\min}\;\mis = 1 - \varepsilon_n,$$
\noindent where $\varepsilon_n$ satisfies
$$cn^{-1/2}C_{k_n}(P) + o(e^{-n}) \leq \varepsilon_n - \sum_{k>k_n}p_k(1-p_k)^n \leq c'n^{-1/2}C_{k_n}(P),$$
\noindent where $c$ and $c'$ are the universal constants from Corollary \ref{cor:cp}.

\end{lemma}

\noindent 
Lemma \ref{lem:cp-without-n} shows that when the number of data $n$ is less than the expected number of samples required to observe $u_k$, the influence of $u_k$ in the computation is tightly related to its probability of appearing. Specifically, its probability of appearing is the probability of a Geometric random variable with parameter $p_k$ of being at most $n$, i.e. exactly $p_k(1-p_k)^n$. When $n\geq 1/p_k$, the influence of $u_k$ in $\underset{\gA}{\min}\;\mis$ changes, and is related to the diversity of the distribution $C_K(P)$.


\begin{proof}[Proof of Lemma \ref{lem:cp-without-n}]
\input{Proofs/Lemma-cpwithoutn}
\end{proof}

\noindent We conclude this section by providing an example in which $\dnu$ has a much faster rate than $n^{-1/2}$.
\begin{lemma} 
\label{lem:exp}
Let $P$ be the Bernoulli distribution with parameter $p\in(0,1)$ and let $\hat{\theta}_n \coloneqq \underset{j}{\sup}\;\rz_j$. Then $$\dnu = (1/2)\max(1,1/\gamma)o((1-p)^n).$$
\end{lemma}

\begin{proof}[Proof of Lemma \ref{lem:exp}]
\input{Proofs/Lemma-exponential}
\end{proof}

\subsection{Relation of $C_K(P)$ with the Gini-Simpson Entropy and the Shannon's Entropy}

We recall that for a discrete random variable $X$ with distribution $P=\sum_{k=1}^{K}p_k\delta_{u_k}$, the Gini-Simpson Entropy is given by (see~\citet{BHARGAVA1974241,RAO198224})

$$
\textrm{G-S}(X) \coloneqq  1 -  \sum\limits_{k=1}^K p_k^2,
$$

\noindent and the Shannon Entropy is given by

$$
 H(X)  \coloneqq - \sum\limits_{k=1}^K p_k \log (p_k).
$$

\noindent From the inequality $\frac{1}{2}\sqrt{p(1-p)}\geq p(1-p)$ for all $p\in[0,1]$, we have 
\begin{equation}
\label{eq:gs1}
\frac{C_K(P)}{2}   = 
\frac{1}{2}   \sum\limits_{k=1}^K \sqrt{p_k(1-p_k)} 
 \geq  \sum\limits_{k=1}^K p_k(1-p_k) =\textrm{G-S}(X).
\end{equation}

\noindent From the concavity of the square root, we also have 

\begin{equation}
\label{eq:gs2}
\frac{C_K(P)}{K}   = 
\frac{1}{K} \sum\limits_{k=1}^K \sqrt{p_k(1-p_k)}  \\
\leq \sqrt{ \frac{1}{K} \sum\limits_{k=1}^K  p_k(1-p_k)} \\
=\sqrt{ \frac{1}{K} \textrm{G-S}(X)}. 
\end{equation}

\noindent The Gini-Simpson index can be interpreted as the expected distance between two randomly selected individuals when the distance is defined as zero if they belong to the same category and one otherwise~\citep{rao1981gini}, that is $\mathbb{P}(X\not=Y)$ for $X$ and $Y$ i.i.d.. The inequality mentioned above suggests that as the Gini-Simpson index increases (e.g., higher diversity of the data), security decreases and thus, the MIAs are expected to be more successful. Another, such commonly used diversity measure is Shannon entropy. Interestingly, $C(P)$ can be also upped and lower bounded by the Shannon entropy as follows: 
\begin{equation}
\label{eq:shannon}
H(X)\leq C_K(P) \leq  \sqrt{K}  \sqrt{H(X)}.
\end{equation}
These bounds easily follow by noticing that 
$$
C_K(P) \leq  \sqrt{K \left[1 - \exp\left(-H(X)\right) \right]} \leq  \sqrt{K H(X)}, 
$$
by upper bounding: 
$$ 
-\log \left(\sum\limits_{k=1}^K p_k^2 \right)\leq  H(X).
$$
Similarly, 
$$
-p_k \log p_k \leq \sqrt{p_k(1-p_k)} , \ \textrm{ for all $0\leq p_k \leq 1,$}
$$
and thus, $C_K(P) \geq  H(X)$.

\noindent The Gini-Simpson Entropy and the Shannon Entropy are maximized by the uniform distribution. This is also the case for $C_K(P)$, as proved below.



\begin{lemma}
\label{lem:cpmax}
Let $P\coloneqq \sum_{k=1}^{K}p_k\delta_{u_k}$ be a discrete distribution with finite $K$. Let $\gM_K$ be the set of all such distributions. We then have the following properties on $C_K(P)$.
For fixed $K\geq 2$, we have
    \begin{itemize}
        \item $\underset{P\in\gM_K}{\max} C_K(P)= \sqrt{K-1}.$
        \item $\underset{P\in\gM_K}{argmax}\ C_K(P) = \textrm{Unif}(u_1,\cdots,u_K)$.
    \end{itemize}
\end{lemma}

\noindent Interestingly, one can observe that for a fixed and finite number of atoms $K$, the sub-levels of $C_K(P)$ are tightly controlled by the value of $\underset{k}{\max}\;p_k$. More precisely, at fixed $\underset{k}{\max}\;p_k\coloneqq\delta$ for some fixed $\delta\in[1/K;1)$, the width of the interval $\left(inf \{ C_K(P) \}, max\{C_K(P) \}\right]$ is entirely determined by $\delta$ and $K$. Specifically, the further $\delta$ is from $1/2$, the thinner the interval gets.  We first summarize this comment below, and give the proof of Lemma \ref{lem:cpmax} after.

\begin{lemma}
\label{lem:cpmaxmax}
    Let $\delta \in [1/K, 1)$. Then the following statements hold :

    \begin{align*}
        \max\,\left\{C_K(P) : P\in\gM_K, \max_k\,p_k = \delta\right\} &= \sqrt{\delta(1-\delta)} + \sqrt{(1-\delta)(K-2+\delta)}, \\ 
        \inf\,\left\{C_K(P) : P\in\gM_K, \max_k\,p_k = \delta\right\}&= \sqrt{\delta(1-\delta)}\lfloor\delta^{-1}\rfloor + \sqrt{\delta\lfloor\delta^{-1}\rfloor(1-\delta\lfloor\delta^{-1}\rfloor)}.
    \end{align*}

\end{lemma}

\begin{proof}[Proof of Lemma \ref{lem:cpmax}]
Denote by $M_K(\delta)\coloneqq \sqrt{\delta(1-\delta)} + \sqrt{(1-\delta)(K-2+\delta)}$. By Lemma \ref{lem:cpmaxmax}, we have that $\underset{P\in\gM_K}{\max} C_K(P) = \underset{\delta\in[1/K,1)}{\max}M_K(\delta) = \sqrt{K-1}$ reached at $\big(\underbrace{\frac{1}{K},\cdots,\frac{1}{K}}_{K}\big)\in\gM_K(1/K)$.
    \end{proof}

\begin{proof}[Proof of Lemma \ref{lem:cpmaxmax}]
Let $\gM_K(\delta) \coloneqq \{P\in\gM_K :\max_k\,p_k = \delta\}$. Without loss of generality, we always assume that $p_K = \max_k\,p_k$. Let $f:p\mapsto\sqrt{p(1-p)}$.\\
First notice that $f$ is a concave function, so that for any $p_1,\cdots,p_m \in [0,1]$ we have $\frac{1}{m}\sum_{k=1}^m f(p_k)\leq f\left(\frac{1}{m}\sum_{k=1}^m p_k\right)$. In particular, for any $P\in\gM_K(\delta)$, we have $\sum_{k=1}^{K-1}p_k = 1-\delta$ giving
\begin{align*}
    C_K(P)\leq C_K\left(\underbrace{\frac{1-\delta}{K-1},\cdots,\frac{1-\delta}{K-1} }_{K-1},\delta\right).
\end{align*}

\noindent Evaluating the r.h.s. of the last inequality gives us the first result.\\

\noindent If $\delta \geq 1/2$, using again the concavity of $f$ gives us that for any $P\in\gM_K(\delta)$, we have 

\begin{align*}
    C_K(P) \geq C_K(\underbrace{0,\cdots,0}_{K-2},1-\delta,\delta),
\end{align*}

\noindent where the last quantity evaluates to $2\sqrt{\delta(1-\delta)}$. \\ 
If $\delta < 1/2$, using the concavity of $f$ gives us that for any $P\in\gM_K(\delta)$, we have

\begin{align*}
    C_K(P) \geq C_K(0,\cdots,0,1-\delta\lfloor\delta^{-1}\rfloor,\underbrace{\delta,\cdots,\delta}_{\lfloor\delta^{-1}\rfloor}),
\end{align*}

\noindent evaluating at $\sqrt{\delta(1-\delta)}\lfloor\delta^{-1}\rfloor + \sqrt{\delta\lfloor\delta^{-1}\rfloor(1-\delta\lfloor\delta^{-1}\rfloor)}$ for any numbers of zeros.
Combining the two results give the last equality of the lemma.
\end{proof}

%% file: Proofs/Lemma-cpwithoutn.tex
From the proof of Corollary \ref{cor:cp}, we know that the minimum is reached in $\gamma=1$, i.e. $\nu = 1/(1+\lambda)$. In particular, by Lemma \ref{lem:corcp-3}, we have

\begin{align*}
   \nonumber \frac{1}{2} \sum_{k=1}^{K}\E\left[\left|\frac{B_k}{n}-p_k\right|\right] &= \frac{1}{n}\sum_{k=1}^{K}\psi(m_k,p_k),
\end{align*}

\noindent where $\psi(m_k, p_k)=\binom{n}{m_k+1}(m_k+1)p_k^{m_k+1}(1-p_k)^{n-m_k}$ with $m_k=\lfloor np_k\rfloor$. For all $k>k_n$, it holds that $m_k=0$ and therefore $\psi(m_k) = p_k(1-p_k)^n$. The result follows from the proof of Corollary \ref{cor:cp}.

%% file: Proofs/Lemma-exponential.tex
From the integral form of $\tilde{D}_\alpha$ (see \eqref{eq:dnu-int}), one has for any distributions $P$ and $Q$,
$$\tilde{D}_\alpha(P,Q) = \frac{1}{2\alpha}\int\left|\alpha p - q\right|d\zeta + \frac{1}{2}\left(1-\frac{1}{\alpha}\right).$$

\noindent Additionally, from \eqref{eq:dnu-cond}, one has 
\begin{align*}
    \tilde{D}_\gamma\left(\mathbb{P}_{(\hat{\theta}_n,\rz_0)}, \mathbb{P}_{(\hat{\theta}_n,\rz_1)}\right) &= \mathbb{E}_{\rz_1}\left[ \frac{1}{2\gamma}\int\left|\gamma p_{\hat{\theta}_n} - p_{\hat{\theta}_n\mid\rz_1}\right|d\zeta + \frac{1}{2}\left(1-\frac{1}{\gamma}\right)\right],
\end{align*}
\noindent where $p_{\hat{\theta}_n}$ (resp. $p_{\hat{\theta}_n\mid\rz_1}$) is the density of $\mathbb{P}_{\hat{\theta}_n}$ (resp. $\mathbb{P}_{\hat{\theta}_n\mid\rz_1}$ conditional to $\rz_1$) with respect to $\zeta$. As the distributions are discrete, for any $b\in\{0,1\}$, one has,

\begin{align*}
    \int\left|\gamma p_{\hat{\theta}_n} - p_{\hat{\theta}_n\mid\rz_1}\right|d\zeta &= \left|\gamma\mathbb{P}\left(\hat{\theta}_n=1\right)- \mathbb{P}\left(\hat{\theta}_n=1\mid\rz_1=b\right)\right| \\ 
    &+ \left|\gamma\mathbb{P}\left(\hat{\theta}_n=0\right)- \mathbb{P}\left(\hat{\theta}_n=0\mid\rz_1=b\right)\right| \\
    &= \left|\gamma(1-(1-p)^n)-\begin{cases}
        1 &\text{, if } b=1 \\ 
        1 - (1-p)^{n-1} & \text{, if } b=0
    \end{cases}\right| \\
    &+ \left|\gamma(1-p)^n-\begin{cases}
        0 &\text{, if } b=1 \\ 
        (1-p)^{n-1} & \text{, if } b=0
    \end{cases}\right| \\
    &= \begin{cases}
        \gamma(1-p)^n + \left|\gamma(1-p)^n - (\gamma-1)\right|& \text{, if }b=1\\
        (1-p)^n|\gamma-1/(1-p)| + \left|\gamma(1-p)^n - (\gamma-1) - (1-p)^{n-1}\right|&\text{, if }b=0
    \end{cases}.
\end{align*}

\noindent Taking the expectation over $\rz_1$ gives 

\begin{align*}
    \mathbb{E}_{\rz_1}\left[\int\left|\gamma p_{\hat{\theta}_n} - p_{\hat{\theta}_n\mid\rz_1}\right|d\zeta\right] &= p\left[\gamma(1-p)^n + \left|\gamma(1-p)^n - (\gamma-1)\right|\right] \\&+ (1-p)\left[(1-p)^n|\gamma-1/(1-p)| + \left|\gamma(1-p)^n - (\gamma-1) - (1-p)^{n-1}\right|\right] \\ 
    &= (1-p)^n\left[\gamma p + (1-p)|\gamma - 1/(1-p)|\right] \\ 
    &+ p|\gamma-1|\left|1-(1-p)^n\frac{\gamma}{\gamma-1}\right| \\ 
    &+ (1-p)|\gamma-1|\left|1 + \frac{(1-p)^{n-1}}{\gamma-1} - (1-p)^n\frac{\gamma}{\gamma-1}\right| \\ 
    &= |\gamma-1| + o((1-p)^n).
\end{align*}

\noindent Combining everything gives
\begin{align*}
    \dnu &= \max(1,\gamma)\left[\tilde{D}_\gamma\left(\mathbb{P}_{(\hat{\theta}_n,\rz_0)}, \mathbb{P}_{(\hat{\theta}_n,\rz_1)}\right) - \lb\right] \\ 
    &= \max(1,\gamma)\left[\frac{1}{2}\left|1-\finv\gamma\right| + \frac{1}{2\gamma}o((1-p)^n) + \frac{1}{2}\left(1-\finv\gamma\right)- \lb\right] \\ 
    &= (1/2)\max(1,1/\gamma)o((1-p)^n),
\end{align*}

\noindent where the last equality comes from the fact that for any $a\in\mathbb{R}$, $$(1/2)(1-a) - (1-a)_+ = -(1/2)|1-a|.$$






%% file: Appendix/v3/J-Numexpsetup-v3.tex
We provide here additional values used for the numerical experiments presented in Section \ref{exp}. All experiments have been conducted with PyTorch library.

\subsection{Details for Section \ref{exp:overfitting}}

\textbf{Non-linear regression.} We consider synthetically generated data $\rx_1,\cdots,\rx_n\overset{i.i.d.}{\sim}\textrm{Unif}(\mathbb{S}_{d-1})$ with $d=10$ and $n=1000$. For any $j=1,\ldots,n$, we set $\ry_j = \Psi^*(\rx_j) + \zeta_j$, where $\zeta_j\sim \gN(0,0.01)$ is some independent noise and $\Psi^*$ is defined for all $x\in\R^d$ by $\Psi^*(x) = \sin(\pi\beta^Tx)$ for some fixed $\beta\in\mathbb{S}_{d-1}$. For the sake of simplicity, we took $\beta = (1,0,\cdots,0)$, as the distribution of the data is rotation-invariant. 
We aim at estimating $\Psi^*$ by a 2 layers ReLU neural network, whose hidden layer has width $4096$. We train the neural network $\Psi_\theta$ by minimizing the MSE loss with the \textit{Adam} optimizer and learning rate $0.1$ for $2500$ iterations. \\
To evaluate the accuracy, we generated $10000$ test samples $x^{test}_1,\cdots,\rx^{test}_{10000}\overset{i.i.d.}{\sim}\textrm{Unif}(\mathbb{S}_{d-1})$ independently from the training dataset.

\noindent \textbf{Linear regression and nearest neighbors}. For both of the simulations, we have varied the dimension of input from $d=10$ up to $d=2000$ with uneven steps between the dimensions. For the nearest neighbors model, we have kept the same data as for the non-linear regression. For the linear regression, we have change $\Psi^*$ simply by $\Psi^*(x) = \beta^Tx$ for the same $\beta$. The nearest neighbor model has been implemented using the class \textbf{KNeighborsRegressor} of the \textit{scikit-learn} library. The learning of the linear regressor has been made similarly to the training of the non-linear regressor.

\subsection{Details for Section \ref{exp:cp}}
We consider here the whole MNIST dataset, and use the given separation between training and test. We aimed at classifying the dataset by learning on a 3-layers ReLU neural network, with both internal layers having width $256$. We trained the neural networks by minimizing the cross-entropy loss with the \textit{Adam} optimizer and learning $0.01$ for 500 iterations. To create the discretizations, we drew three times $1000$ samples from the dataset (that have been used only for that purpose) and we constructed the clusterizing learning procedure based on the \textbf{MiniBatchKMeans} function from \textit{scikit-learn} library. We used $K=20,50,100$ clusters respectively. For each dataset size $n=1000,5000,10000$, we performed the following steps :

\begin{itemize}
    \item We draw a dataset $D_n$ of size $n$ not containing the samples used for the clusterings.
    \item We trained a first neural network on $D_n$ (column \textbf{raw dataset}).
    \item For each clustering, we discretized $D_n$ and then trained a neural network on the new dataset.
\end{itemize}



%% file: Appendix/v3/A-proofforPerformances-v3.tex


We give in this section the proofs relative to Section \ref{main_res_perf}.\\

\begin{proof}[Proof of Proposition \ref{prop:dnu-props}]
\input{Proofs/Proposition-propsonDelta}
\end{proof}

\begin{proof}[Proof of Theorem \ref{thm:dnu}]
\input{Proofs/Theorem-general}
\end{proof}

%% file: Proofs/Proposition-propsonDelta.tex
Let $p$ (resp. $q$) be the density of $P$ (resp. $Q$) with respect to any dominating measure $\zeta$. Let $B^* = \{\alpha p\geq q\}$ be the set reaching the supremum in the definition of $\tilde D_{\gl}$, namely

\begin{align*}
\tilde{D}_\alpha(P,Q) &= \finv\alpha\int_{B^*}\left(\alpha p-q\right)d\zeta.
\end{align*}

\noindent By taking the complementary in the supremum, we see that 

\begin{align*}
    \tilde{D}_\alpha(P,Q)&=\finv\alpha\underset{B}{\sup}\;Q(B) - \alpha P(B) + \left(1-\finv\alpha\right) \\ 
    &= \finv\alpha\int_{(B^*)^c}(q-\alpha p)d\zeta+ \left(1-\finv\alpha\right).
\end{align*}
\noindent Taking the average of both gives,

\begin{equation}
\label{eq:dnu-int}
\tilde{D}_\alpha(P,Q) = \frac{1}{2\alpha}\int\left|\alpha p - q\right|d\zeta + \frac{1}{2}\left(1-\frac{1}{\alpha}\right),
\end{equation}

\noindent from which we deduce

\begin{align}
  \label{eq:D-int}  D_\alpha(P,Q) &= \max(1,\alpha)\left[\frac{1}{2\alpha}\int\left|\alpha p - q\right|d\zeta + \frac{1}{2}\left(1-\frac{1}{\alpha}\right) - \left(1-\finv\alpha\right)_+\right] \\ 
  \label{eq:D-int-f}  &= \int f_\alpha(p/q)qd\zeta,
\end{align}

\noindent where the second equality comes from $(1/2)(1-1/\alpha) - (1-1/\alpha)_+ = -(1/2)|1-1/\alpha|$.

\noindent Equation~\ref{eq:dnu-cond} follows immediately from equations \ref{eq:dnu-int} and \ref{eq:D-int}.\\
\noindent From \eqref{eq:D-int-f}, $f_\alpha$ being convex, continuous and satisfies $f_\alpha(1) = 0$, $D_\alpha$ is an $f-$divergence.\\
\noindent We now show the inequalities. We only prove the lower bound as the upper bound is trivial. By definition of $\tilde{D}_\alpha$, for any measurable set $B$, we have,

\begin{align*}
    \tilde{D}_\alpha(P,Q) &\geq \finv\alpha\left(\alpha P(B) -Q(B)\right).
\end{align*}

\noindent Then taking $B$ as the whole space, we get,

\begin{align*}
    \tilde{D}_\alpha(P,Q) &\geq  \finv\alpha\left(\alpha-1\right)= 1 -\finv\alpha.
\end{align*}

\noindent Additionally, taking $B$ as the null set, we have trivially $\tilde{D}_\alpha \geq 0$, which implies $\tilde{D}_\alpha(P,Q)\geq \left(1-\finv\alpha\right)_+$, which leads to the lower bound.

%% file: Proofs/Theorem-general.tex
\\ \noindent By the \textit{i.i.d.} assumption of the data, and by the independence of $T$ to the rest of the random variables, we have 

\begin{align*}
    \acc &= \nu\mathbb{P}\left(\phi(\hat{\theta}_n,\rz_0)=0\right) + \lambda(1-\nu)\mathbb{P}\left(\phi(\hat{\theta}_n,\rz_1)=1\right).
\end{align*}


%

\noindent We now define $B \coloneqq \left\{(\theta,z)\in\Theta\times\gZ : \phi(\theta,z)=1\right\}$ and rewrite ${\text{Acc}}_n(\phi;P,\gA)$ as

\begin{align}
\label{eq:perff}
{\text{Acc}}_n(\phi;P,\gA) &= \nu\left(1-\mathbb{P}\left((\hat{\theta}_n,{\rz_0})\in B\right)\right) + \lambda(1-\nu) \mathbb{P}\left((\hat{\theta}_n,\rz_1)\in B\right). 
\end{align}

\noindent Taking the maximum over all MIAs $\phi$ then reduces to taking the maximum of the r.h.s. of \eqref{eq:perff} over all measurable sets $B$. We then get 

\begin{align}
\label{eq:max1}
\underset{\phi}{\max}\;{\text{Acc}}_n(\phi;P,\gA) &=\lambda(1-\nu)\underset{B}{\max}\Bigl[\mathbb{P}\left((\hat{\theta}_n,{\rz_1})\in B\right) - \gamma \mathbb{P}\left((\hat{\theta}_n,\rz_0)\in B\right)\Bigr] + \nu,
\end{align}

\noindent where the maximum is taken over all measurable sets $B$. Let now $\zeta$ be a dominating measure of the distributions of $(\hat{\theta}_n,{\rz_0})$ and  $(\hat{\theta}_n,\rz_1)$ (for instance their average). We denote by $p$ (resp. $q$) the density of the distribution $\mathbb{P}_{(\hat{\theta}_n,{\rz_0})}$ of $(\hat{\theta}_n,{\rz_0})$ (resp. $\mathbb{P}_{(\hat{\theta}_n,{\rz_1})}$ for $(\hat{\theta}_n,\rz_1)$) with respect to $\zeta$. Then, the involved maximum in the r.h.s. of \eqref{eq:max1} is reached on the set $${B^* \coloneqq\{q/p\geq\gamma\}}.$$
The maximum being taken over all measurable sets in \eqref{eq:max1}, we may consider replacing $B$ by its complementary $B^c$ in the expression giving 

\begin{align}
\label{eq:max2}
\underset{\phi}{\max}\;{\text{Acc}}_n(\phi;P,\gA) &= \lambda(1-\nu)\underset{B}{\max}\;\Bigl[\gamma \mathbb{P}\left((\hat{\theta}_n,\rz_0)\in B\right) - \mathbb{P}\left((\hat{\theta}_n,{\rz_1})\in B\right)\Bigr] +\lambda(1-\nu),
\end{align}
which comes from the definition of $\gamma$. In this case the maximum is reached on the set $${B^*}^c \coloneqq \{q/p<\gamma\}.$$ 
In particular, by the definition of $\tilde D_\alpha$, we then get,
$$\underset{\phi}{\max}\;{\text{Acc}}_n(\phi;P,\gA) = \lambda(1-\nu) + \nu\tilde D_{\gamma}\left(\mathbb{P}_{(\hat{\theta}_n,\rz_0)},\mathbb{P}_{(\hat{\theta}_n,\rz_1)}\right).$$

\noindent Therefore, if $\lambda(1-\nu) \geq \nu$, it holds $\gamma\leq1$, and in particular $\tilde D_{\gamma}\left(\mathbb{P}_{(\hat{\theta}_n,\rz_0)},\mathbb{P}_{(\hat{\theta}_n,\rz_1)}\right) = \dnu$. Therefore, we have

\begin{align*}
    \underset{\phi}{\max}\;{\text{Acc}}_n(\phi;P,\gA) - \lambda(1-\nu) &= \nu\tilde D_{\gamma}\left(\mathbb{P}_{(\hat{\theta}_n,\rz_0)},\mathbb{P}_{(\hat{\theta}_n,\rz_1)}\right)\\ 
    &= \nu\dnu.
\end{align*}

\noindent On the converse, if $\lambda(1-\nu)\leq\nu$, it holds $\gamma\geq1$, and we have

\begin{align*}
    \underset{\phi}{\max}\;{\text{Acc}}_n(\phi;P,\gA) - \nu &= \lambda(1-\nu) + \nu\left(\tilde D_{\gamma}\left(\mathbb{P}_{(\hat{\theta}_n,\rz_0)},\mathbb{P}_{(\hat{\theta}_n,\rz_1)}\right)-1\right) \\ 
    &= \lambda(1-\nu) + \nu\left(\frac{1}{\gamma}\dnu + \left(1-\frac{1}{\gamma}\right)-1\right) \\ 
    &=\nu\frac{1}{\gamma}\dnu,
\end{align*}

\noindent hence the inequality.\\
\noindent Now, by the definition of the MIS, we have 
\begin{align*}
    \mis&=\finv{\min(\nu,\lambda(1-\nu))}\left(\nu + \lambda(1-\nu)-\left[\lambda(1-\nu) + \nu\tilde D_{\gamma}\left(\mathbb{P}_{(\hat{\theta}_n,\rz_0)},\mathbb{P}_{(\hat{\theta}_n,\rz_1)}\right)\right]\right) \\ 
    &= \finv{\min(1,1/\gamma)}\left(1-\tilde D_{\gamma}\left(\mathbb{P}_{(\hat{\theta}_n,\rz_0)},\mathbb{P}_{(\hat{\theta}_n,\rz_1)}\right)\right) \\ 
    &= 1 -  \finv{\min(1,1/\gamma)}\left(\min(1,1/\gamma) - 1 +\tilde D_{\gamma}\left(\mathbb{P}_{(\hat{\theta}_n,\rz_0)},\mathbb{P}_{(\hat{\theta}_n,\rz_1)}\right)\right) \\ 
    &= 1 -  \finv{\min(1,1/\gamma)}\left( \tilde D_{\gamma}\left(\mathbb{P}_{(\hat{\theta}_n,\rz_0)},\mathbb{P}_{(\hat{\theta}_n,\rz_1)}\right) - \lb\right) \\ 
    &= 1 - \dnu,
\end{align*}
\noindent hence the equality.

%% file: Appendix/v3/D-proofforOverfitting-v3.tex
We give here the proofs for the Section \ref{overfitting}. \\

\begin{proof}[Proof of Proposition \ref{prop:overfitting}]
\input{Proofs/Proposition-overfitting}
\end{proof}

\begin{proof}[Proof of Theorem \ref{thm:overfitting}]
\input{Proofs/Theorem-overfitting}    
\end{proof}

%% file: Proofs/Proposition-overfitting.tex
Let $l_j \coloneqq l_j(\hat{\theta}_n) = l_{\hat{\theta}_n}(\rx_j,\ry_j)$. The learning procedure $\gA_{\varepsilon,\alpha}$ stops as soon as $\frac{1}{n}\sum_{j=1}^{n}l_{\hat{\theta}_n}(\rx_j,\ry_j) \leq \varepsilon\alpha$.\\
Let $B_\varepsilon\coloneqq\left\{j : l_{\hat{\theta}_n}(\rx_j,\ry_j)\leq\varepsilon\right\}$ be the set of samples with loss not larger than $\varepsilon$ at the end of the training. We then have the following sequence of inequalities:

\begin{align*}
    n\varepsilon\alpha \geq \sum_{j=1}^n l_{\hat{\theta}_n}(\rx_j,\ry_j) \geq \sum_{j\in B_\varepsilon}l_{\hat{\theta}_n}(\rx_j,\ry_j) + (n-\#B_\varepsilon)\varepsilon \geq (n-\#B_\varepsilon)\varepsilon.
\end{align*}

\noindent From the two extremes, we get that $\sum_{j=1}^{n}1\{l_j\leq\varepsilon\}\coloneqq\#B_\varepsilon \geq n(1-\alpha)$. From the $i.i.d.$ hypothesis on the data $\rz_1,\cdots,\rz_n$, taking the expectation gives the result.

%% file: Proofs/Theorem-overfitting.tex
Let $S_{\theta}^{\varepsilon} := \{(x,y)\in\gX\times\gY : l_{\theta}(x,y)\leq\varepsilon\}$ be the $\varepsilon$-sub-level set of $l_\theta$ for all $\theta\in\Theta$.
We begin by proving the first point. Let $\gA$ be an $(\varepsilon,1-\alpha)$-overfitting learning procedure, and let $S^{\varepsilon} \coloneqq \{(\theta,x,y) : l_{\theta}(x,y) \leq\varepsilon\}$. From the definition of $\tilde D_\gl$, and by taking the complementary in the supremum, we have that 
\begin{align*}
\tilde{D}_\gl(\mathbb{P}_{(\hat{\theta}_n,\rz_0)},\mathbb{P}_{(\hat{\theta}_n,\rz_1)}) &\geq \flg\mathbb{P}((\hat{\theta}_n,\rx_1,\ry_1)\in S^{\varepsilon}) -  \mathbb  {P}((\hat{\theta}_n,\rx,\ry)\in S^{\varepsilon}) + \left(1-\flg\right) \\ 
&= \flg\mathbb{P}((\rx_1, \ry_1)\in S_{\hat{\theta}_n}^{\varepsilon}) - \mathbb{P}((\rx,\ry)\in S_{\hat{\theta}_n}^{\varepsilon}) + \left(1-\flg\right)\\ 
&\geq \flg(1 - \alpha) -  \mathbb{P}((\rx, \ry)\in S_{\hat{\theta}_n}^{\varepsilon})+ \left(1-\flg\right) \\
&= 1-\flg\alpha - \int_{\theta\in\Theta} P((\rx, \ry)\in S_{\theta}^{\varepsilon})d\mu_{\hat{\theta}_n},
\end{align*}
\noindent which proves the first point by 

\begin{align*}
    \mis &= 1 - \dnu \\
    &= 1 - \max(1,\gl)\left(\tilde{D}_\gl(\mathbb{P}_{(\hat{\theta}_n,\rz_0)},\mathbb{P}_{(\hat{\theta}_n,\rz_1)})-\lb\right) \\ 
    &= \begin{cases}
        1 - \gamma\left[1-\finv\gamma\alpha - \mathbb{P}\left((\rx,\ry)\in S_{\hat{\theta}_n}^\varepsilon\right) - \left(1-\finv\gamma\right) \right] & \text{, if } \gamma\geq1 \\ 
        1 - \left[1 - \finv\gamma\alpha - \mathbb{P}\left((\rx,\ry)\in S_{\hat{\theta}_n}^\varepsilon\right)\right] & \text{, if } \gamma\leq1
    \end{cases}\\
    &= \begin{cases}
        \alpha + \gamma\mathbb{P}\left((\rx,\ry)\in S_{\hat{\theta}_n}^\varepsilon\right) & \text{, if } \gamma \geq1 \\ 
        \finv\gamma\left(\alpha + \gamma\mathbb{P}\left((\rx,\ry)\in S_{\hat{\theta}_n}^\varepsilon\right)\right) & \text{, if } \gamma \leq1
    \end{cases}\\
    &\leq \max(1,\lg)\left(\alpha + \fgl\mathbb{P}\left((\rx,\ry)\in S_{\hat{\theta}_n}^\varepsilon\right)\right).
\end{align*}

\noindent Now assume that we have a sequence of learning procedures $(\gA_{\eta})_{\eta\in\R^+}$ that stop as soon as $L_n\leq\eta$. Assume the additional hypotheses given in the second point of Theorem \ref{thm:overfitting} hold. Let $\alpha\in(0,1)$ be a fixed scalar. By Proposition \ref{prop:overfitting}, $\gA_\eta$ is $(\eta/\alpha,1-\alpha)$-overfitting, so that by the first point proven above, we have 

$$\text{Sec}_{\nu,\lambda,n}(P,\gA_\eta) \leq \max(1,\lg)\left(\alpha + \fgl\mathbb{P}\left((\rx,\ry)\in S_{\hat{\theta}_n}^{\eta/\alpha}\right)\right).$$

\noindent For any $\theta\in\Theta$, we have 

\begin{align*}
    P((\rx,\ry)\in S_\theta^{\eta/\alpha}) &= \E\left[P((\rx,\ry)\in S_\theta^{\eta/\alpha}\mid\rx)\right] \\ 
    &= \E\left[P(\omega(y,\Psi_\theta(x))\leq\eta/\alpha\mid\rx)\right] \\ 
    &\overset{\eta\to0^+}{\to} 0,
\end{align*}

\noindent where the limit comes from the continuity of $\omega$ and the absolute continuity of the distribution of $\ry$ given $\rx$. In particular, for any $\alpha\in(0,1)$, we then have 

$$\underset{\eta\to0}{\lim}\;\text{Sec}_{\nu,\lambda,n}(P,\gA_\eta)\leq\max(1,1/\gamma)\alpha.$$

\noindent Taking the infimum over $\alpha$ gives the result.


%% file: Appendix/v3/B-proofforMeans-Discrete-v3.tex
We give here the proofs for the Section \ref{main_res_disc}.

\begin{proof}[Proof of Theorem \ref{theo:empmean}]
\input{Proofs/Theorem-empiricalmean}
\end{proof}

\begin{proof}[Proof of Proposition \ref{prop:empmeangauss}]
\input{Proofs/Proposition-empiricalmeangaussian}
\end{proof}

\begin{proof}[Proof of Remark \ref{theo:empmean}]
\label{proof:mean:minn}
\input{Proofs/Remark-numberofsamplesinempiricalmean}
\end{proof}

\noindent We recall from Proposition \ref{prop:dnu-props} that the following equation holds
\begin{equation*}
\dnu = \mathbb{E}_{\rz_1}\left[{D}_\gamma\left(\mathbb{P}_{\hat{\theta}_n},\mathbb{P}_{\hat{\theta}_n\mid\rz_1}\right)\right].   
\end{equation*}

\begin{proof}[Proof of Theorem \ref{theo:disc}]
\input{Proofs/Theorem-discrete}
\end{proof}

\begin{proof}[Proof of Lemma \ref{lem:deltamax}] 
\input{Proofs/Lemma-discretedeltamax}
\end{proof}

\noindent To prove Corollary \ref{cor:cp}, we need the following intermediary results.

\begin{lemma}
\label{lem:corcp-1}
The function $f^+:x\mapsto \sum_{k=1}^{K}\mathbb{E}\left[\left|x\frac{B_k}{n}-p_k\right|\right] - (x-1)$ is non-increasing on $[1,\infty).$
\end{lemma}

\begin{lemma}
\label{lem:corcp-2}
For any $t,p\in\mathbb{R}^+$, the function $f_{t,p}^-:x\mapsto\frac{1}{x}|x t-p| - p\left(\frac{1}{x}-1\right)$ is non-decreasing on $(0,1]$.
\end{lemma}

\begin{lemma}
\label{lem:corcp-3}
Define for any $m=0,\ldots,n$ and any $p\in(0,1)$ the function $\psi(m,p)=\binom{n}{m+1}(m+1)p^{m+1}(1-p)^{n-m}$. Then letting $m_k = \lfloor\fgl np_k\rfloor$, it holds that 

\begin{equation}
     \frac{1}{2\gamma} \sum_{k=1}^{K}\E\left[\left|\frac{B_k}{n}-\fgl p_k\right|\right] =   \frac{1}{\gamma n}\sum_{k=1}^{K}\psi(m_k,p_k)+ \frac{1}{2}\left|1-\flg\right| - \left|1-\flg\right|o(e^{-n}).
\end{equation}   
\end{lemma}

\begin{lemma}
\label{lem:corcp-4}
Let $m = \lfloor np\rfloor$ for some probability $p\in(0,1)$ such that $1\leq m\leq n-2$.

\noindent Then there exist universal constants $c>0.29$ and $c'<0.44$ such that, 

\begin{equation}
    c\sqrt{n}\sqrt{p(1-p)}\leq\psi(m,p) \leq c'\sqrt{n}\sqrt{p(1-p)}.
\end{equation}

\end{lemma}

\noindent Additionally, to prove Lemma \ref{lem:corcp-3}, we need a result from \cite{DeMoivre} which we state and prove at the end of the section for completion.

\begin{lemma}[\cite{DeMoivre}]
\label{lem:corcp-5}
Let $n\geq1$ be an integer, and $p\in(0,1)$ be a real number. Let $m=0,\cdots,n$ be any integer. Then it holds that 

\begin{equation}
\label{eq:demoivre-lem}
\sum_{k=0}^m\binom{n}{k}p^k(1-p)^{n-k}(np-k) = (m+1)\binom{n}{m+1}p^{m+1}(1-p)^{n-m},
\end{equation}
\noindent where $\binom{n}{k}$ is the binomial coefficient. For generality, we set $\binom{n}{n+1}=0$.
\end{lemma}

\begin{proof}[Proof of Corollary \ref{cor:cp}]
\input{Proofs/Corollary-discrete}
\end{proof}

\begin{proof}[Proof of Lemma \ref{lem:corcp-1}]

\input{Proofs/Lemma-cp1}
\end{proof}

\begin{proof}[Proof of Lemma \ref{lem:corcp-2}]
\input{Proofs/Lemma-cp2}
\end{proof}

\begin{proof}[Proof of Lemma \ref{lem:corcp-3}]
\input{Proofs/Lemma-cp3}
\end{proof}

\begin{proof}[Proof of Lemma \ref{lem:corcp-4}]
\input{Proofs/Lemma-cp4}
\end{proof}

\begin{proof}[Proof of Lemma \ref{lem:corcp-5}]
\input{Proofs/Lemma-cp5}
\end{proof}

%% file: Proofs/Theorem-empiricalmean.tex
\noindent First, note that we have for any distributions $P$ and $Q$,
\begin{align*}
    \\\tilde D_\alpha(P,Q) &= \finv\alpha\underset{B}{\sup}\; \left[\alpha P(B) - \alpha Q(B) + (\alpha-1)Q(B) \right]\\ 
    &\leq \begin{cases}
        \|P-Q\|_{TV}&\text{, if }\alpha<1\\
        \|P-Q\|_{TV} + \left(1-\finv\alpha\right)&\text{, if }\alpha\geq1
    \end{cases}\\
    &= \|P-Q\|_{TV} + \left(1-\finv\alpha\right)_+.
\end{align*}

\noindent By Theorem \ref{thm:dnu}, we have
\begin{align*}
    \mis &\geq 1 - \max(1,\gl)\left\|\mathbb{P}_{\left(\hat{\theta}_n,\rz_0\right)}-\mathbb{P}_{\left(\hat{\theta}_n,\rz_1\right)}\right\|_{TV}.
\end{align*}
\noindent We will therefore prove an upper bound on $\left\|\mathbb{P}_{\left(\hat{\theta}_n,\rz_0\right)}-\mathbb{P}_{\left(\hat{\theta}_n,\rz_1\right)}\right\|_{TV}$.
\noindent For any positive integer $j$, we let ${m_j \coloneqq \E\left[\left\|C^{-1/2}\Bigl\{L(\rz_1)-\E 
\left[L(\rz_1)\right]\Bigr\}\right\|_2^j\right]}$ be the expectation of the $j$-th power of the norm of the centered and reduced version of $L(\rz_1)$, where $C$ is the covariance matrix of $L(\rz_1)$.

\noindent Setting $L_n := \frac{1}{n}\sum_{j=1}^{n}L(\rz_j)$, by the data processing inequality \citep{DPT1973} applied to the total variation distance, for any measurable map $g:\R^d\times\gZ\to\gZ'$ taking values in any measurable space $\gZ'$, we have
\begin{equation*}
    \left\|\gL(g(L_n,\rz_1))- \gL(g(L_n,\rz_0))\right\|_{\textrm{TV}} \leq \left\|\gL((L_n,\rz_1))- \gL((L_n,\rz_0))\right\|_{\textrm{TV}}.
\end{equation*}
\noindent The inequality holds
in particular  for $g$ defined for all $(l,z)$ in $\R^d\times\gZ$ by $g(l,z) = (F(l),z)$, from which we get 

\begin{align*}
    \left\|\mathbb{P}_{\left(\hat{\theta}_n,\rz_1\right)}-\mathbb{P}_{\left(\hat{\theta}_n,\rz_0\right)}\right\|_{TV}&\leq \left\|\gL((L_n,\rz_1))- \gL((L_n,\rz_0))\right\|_{\textrm{TV}}\\&=\mathbb{E}\left[\left\|\gL(L_n\mid\rz_1)-\gL(L_n)\right\|_{\textrm{TV}}\right],
\end{align*}
\noindent in which the expectation is taken over $\rz_1$. Here, for any random variable $\rx$, $\gL(\rx)$ denotes its distribution. \\
\noindent For $j=1,\ldots, n$, denote by $\rv_j \coloneqq C^{-1/2}(L(\rz_j) - \E[L(\rz_j)])$ the centered and reduced version of $L(\rz_j)$. The total variation distance being invariant by translation and rescaling, we shall write 
\begin{align*}
 \left\|\gL(L_n\mid\rz_1)-\gL(L_n)\right\|_{\textrm{TV}} & = \Big\|\gL(L_n - \E[L(\rz_1)])-\gL(L_n - \E[L(\rz_1)]\mid\rz_1)\Big\|_{\textrm{TV}} \\ 
    &= \left\|\gL\left(\frac{1}{n}\sum_{j=1}^{n}(L(\rz_j)-\E[L(\rz_j)])\right)-\gL\left(\frac{1}{n}\sum_{j=1}^{n}(L(\rz_j)-\E[L(\rz_j)])\Bigm\vert\rz_1\right)\right\|_{\textrm{TV}} \\ 
    &= \left\|\gL\left(\frac{C^{-1/2}}{\sqrt{n}}\sum_{j=1}^{n}(L(\rz_j)-\E[L(\rz_j)])\right)-\gL\left(\frac{C^{-1/2}}{\sqrt{n}}\sum_{j=1}^{n}(L(\rz_j)-\E[L(\rz_j)])\Bigm\vert\rz_1\right)\right\|_{\textrm{TV}} \\ 
    &= \left\|\gL\left(\frac{1}{\sqrt{n}}\sum_{j=1}^{n}\rv_j\right)-\gL\left(\frac{1}{\sqrt{n}}\sum_{j=1}^{n}\rv_j\Bigm\vert\rv_1\right)\right\|_{\textrm{TV}}.
\end{align*}

\noindent Denoting by $\gN_d(\beta,\Sigma)$ the $d-$dimensional normal distribution with parameters $(\beta,\Sigma)$, it holds almost surely that
\begin{align*}
    \left\|\gL\left(\frac{1}{\sqrt{n}}\sum_{j=1}^{n}\rv_j\right)-\gL\left(\frac{1}{\sqrt{n}}\sum_{j=1}^{n}\rv_j\Bigm\vert\rv_1\right)\right\|_{\textrm{TV}} 
    &\leq \left\|\gL\left(\frac{1}{\sqrt{n}}\sum_{j=1}^{n}\rv_j\right) - \gN_d(0,\mI_d)\right\|_{\textrm{TV}} \\ 
    &+ \left\|\gL\left(\frac{1}{\sqrt{n}}\sum_{j=1}^{n}\rv_j\Bigm\vert \rv_1\right) - \gN_d\left(\frac{1}{\sqrt{n}}\rv_1,\frac{n-1}{n}I_d\right)\right\|_{\textrm{TV}} \\ 
    &+ \left\|\gN_d\left(0,\mI_d\right) - \gN_d\left(\frac{1}{\sqrt{n}}\rv_1,\frac{n-1}{n}\mI_d\right)\right\|_{\textrm{TV}} \\ 
    &= \left\|\gL\left(\frac{1}{\sqrt{n}}\sum_{j=1}^{n}\rv_j\right) - \gN_d(0,\mI_d)\right\|_{\textrm{TV}} \\ 
    &+ \left\|\gL\left(\frac{1}{\sqrt{n-1}}\sum_{j=1}^{n-1}\rv_j\right) - \gN_d(0,\mI_d)\right\|_{\textrm{TV}} \\ 
    &+ \left\|\gN_d\left(0,\mI_d\right) - \gN_d\left(\frac{1}{\sqrt{n}}\rv_1,\frac{n-1}{n}\mI_d\right)\right\|_{\textrm{TV}}.
\end{align*}

\noindent Applying Theorem 2.6 of \citet{cltTV2016} with variable $\rv_j$ and parameter $r=2$, one can upper bound the first two terms by some constant $C(d)(1+m_3)$ times $n^{-1/2}$. The constant $C(d)$ here depends only on the dimension of the parameters $d$. We may upper bound the last term by the following proposition.

\begin{proposition}
\label{prop:empmeangauss}   
Let $n$ be an integer and $\beta\in\R^d$ be any $d-$dimensional vector. Then it holds that 
\begin{align*}
\left\|\gN_d(0,\mI_d) - \gN_d\left(\frac{1}{\sqrt{n}}\beta,\frac{n-1}{n}\mI_d\right)\right\|_{\textrm{TV}} \leq \frac{\sqrt{d}}{2n} + \frac{1}{2\sqrt{n}}\|\beta\|_2. 
\end{align*}
\end{proposition}

\noindent Applying Proposition \ref{prop:empmeangauss} to the last quantity, it holds that 
$$\left\|\gN_d\left(0,\mI_d\right) - \gN_d\left(\frac{1}{\sqrt{n}}\rv_1,\frac{n-1}{n}\mI_d\right)\right\|_{\textrm{TV}} \leq \frac{\sqrt{d}}{2n} + \frac{1}{2\sqrt{n}}\|\rv_1\|_2,$$
and the result follows from taking the expectation, with $c_{L,P}=C(d)(1+m_3) + \frac{m_1}{2}$.

%% file: Proofs/Proposition-empiricalmeangaussian.tex
\noindent Applying Proposition 2.1 of \citet{tvgauss2018}, it holds almost surely that 
\begin{align*}
&\left\|\gN_d(0,\mI_d) - \gN_d\left(\frac{1}{\sqrt{n}}\beta,\frac{n-1}{n}\mI_d\right)\right\|_{\textrm{TV}} \\
&\leq \frac{1}{2}\sqrt{tr\left(\mI_d\frac{n-1}{n}\mI_d-\mI_d\right) + \frac{1}{n}\|\beta\|_2^2-\ln\left(det\left(\frac{n-1}{n}\mI_d\right)\right)}  \\ 
&= \frac{1}{2}\sqrt{-\frac{d}{n} + \frac{1}{n}\|\beta\|_2^2 - d\ln\left(\frac{n-1}{n}\right)} \\ 
&\leq\frac{1}{2}\sqrt{-d\left(\frac{1}{n}+\ln\left(\frac{n-1}{n}\right)\right)}+\frac{1}{2}\sqrt{\frac{1}{n}\|\beta\|_2^2} \\ 
&\leq \frac{\sqrt{d}}{2n} + \frac{1}{2\sqrt{n}}\|\beta\|_2,
\end{align*}
where $tr(\cdot)$ is the trace operator and $det(\cdot)$ is the matrix determinant operator. The third inequality is due to ${\sqrt{a+b}\leq\sqrt{a}+\sqrt{b}}$ for positive scalars $a$ and $b$. The first term in the last inequality comes from the fact that ${x-1-\ln(x)\leq(x-1)^2}$ if $x\geq 1/3$ which holds with $x=\frac{n-1}{n}$ for $n\geq 2$.

%% file: Proofs/Remark-numberofsamplesinempiricalmean.tex
From \eqref{eq:empmean}, we have that
\begin{align*}
    cn^{-1/2} + \frac{\sqrt{d}}{2}n^{-1}\leq\varepsilon,
\end{align*}
is sufficient to ensure $\left\|\mathbb{P}_{\left(\hat{\theta}_n,\rz_1\right)}-\mathbb{P}_{\left(\hat{\theta}_n,\rz_0\right)}\right\|_{TV}\leq\varepsilon$, hence a security of at least $1-\max(1,\gl)\varepsilon$.

\noindent Setting $x\coloneqq n^{-1/2}$, it is equivalent to
$$cx + \frac{\sqrt{d}}{2}x^2-\varepsilon\leq0.$$

\noindent From the the study of the above quadratic function, as $x\geq 0$ is assumed, we get that this is equivalent to

\begin{align*}
    n^{-1/2}&\leq \frac{-c + \sqrt{c^2 + 2\varepsilon\sqrt{d}}}{\sqrt{d}} \\ 
    \iff n &\geq \frac{d}{2c^2 + 2\varepsilon\sqrt{d}-2c\sqrt{c^2+2\varepsilon\sqrt{d}}} \\ 
    &= \frac{d}{2c^2}\frac{1}{1+\frac{\varepsilon\sqrt{d}}{c^2} - \sqrt{1 + 2\frac{\varepsilon\sqrt{d}}{c^2}}}.
\end{align*}

\noindent From the mean-value form of Taylor theorem of order $2$ at $0$, there exists $0\leq\Bar{u}\leq u\coloneqq \frac{\varepsilon\sqrt{d}}{c^2}$ such that 

$$\sqrt{1+2u} = 1 + u - \frac{1}{2}(1+2\Bar{u})^{-3/2}.$$

\noindent Therefore, the condition becomes 

\begin{align*}
    n &\geq \frac{d}{2c^2}\frac{2(1+2\Bar{u})^{3/2}}{u^2} \\ 
    &= \varepsilon^{-2}c^2(1+2\Bar{u})^{3/2}.
\end{align*}

\noindent As $\Bar{u} \leq u \leq \frac{\sqrt{d}}{c^2}$, $n \geq \varepsilon^{-2}c^2(1+\frac{\sqrt{d}}{c^2})^{3/2}$ ensures the above condition, hence the result.

%% file: Proofs/Theorem-discrete.tex
The proof will be divided in two steps. First, we will prove the inequality 
\begin{equation}
\label{eq:deltabornesup}
\dnu \leq 
\max(1,\gl)\left[\frac{1}{2\gamma}\sum_{k=1}^{K}\E\left[\left|\frac{B_k}{n} - \fgl p_k\right|\right] - \frac{1}{2}\left|1-\flg\right|\right],
\end{equation}

\noindent for any distribution $P$ and learning procedure $\gA$. Second, we prove that this upper bound is reached for learning procedures that map any data set to a Dirac mass, summarized in the following lemma.

\begin{lemma}
\label{lem:deltamax}
For $k=1,\ldots,K$, let $B_k$ be random variables having Binomial distribution with parameters $(n,p_k)$. Suppose that $\gA(z_1,\cdots,z_n) = \delta_{F\left(\frac{1}{n}\sum_{j=1}^n\delta_{z_j}\right)}$ for any $n\in\mathbb{N}$ and $z_1,\ldots,z_n\in\gZ$, for some measurable map ${F:\gM\to\Theta}$ with infinite range $|\Theta|=\infty$, i.e. $\hat{\theta}_n \overset{\gL}{=}F\left(\frac{1}{n}\sum_{j=1}^n\delta_{\rz_j}\right)$. Then we have 

$$\underset{F}{\max}\;\dnu =  \max(1,\gl)\left[\frac{1}{2\gamma}\sum_{k=1}^{K}\E\left[\left|\frac{B_k}{n} - \fgl p_k\right|\right] - \frac{1}{2}\left|1-\flg\right|\right].$$
\end{lemma}

\noindent Lemma \ref{lem:deltamax} and \eqref{eq:deltabornesup} will imply Theorem \ref{theo:disc} as follows,

\begin{align*}
    \underset{\gA}{\min}\;\mis &= 1 - \underset{\gA}{\max}\;\dnu \\ 
    & = 1 - \max(1,\gl)\left[\frac{1}{2\gamma}\sum_{k=1}^{K}\E\left[\left|\frac{B_k}{n} - \fgl p_k\right|\right] - \frac{1}{2}\left|1-\flg\right|\right].
\end{align*}

 
\noindent We first prove \eqref{eq:deltabornesup}.\\ 
Since $\hat{\theta}_n$ has distribution $G(\hat{P}_n)$ conditionally on $\rvz$, where $\hat{P}_n \coloneqq \frac{1}{n}\sum_{j=1}^{n}\delta_{\rz_j}$ is the empirical distribution of the data set, from Proposition \ref{prop:G}, we have

\begin{align}
\begin{split}
\label{eq:PB}
    \mathbb{P}(\hat{\theta}_n\in B) &= \E[\mathbb{P}(\hat{\theta}_n\in B\vert\rvz)] \\ 
    &=\E[G(\hat{P}_n)(B)]
\end{split}  \\
\label{eq:PBcond}
    \mathbb{P}(\hat{\theta}_n\in B\vert\rz_1) &= \E[G(\hat{P}_n)(B)\vert\rz_1],
\end{align}
for any measurable set $B$.\\
Recall that $u_{1},\ldots,u_{K}$ are the (fixed) support points of $P$. For any $k\in\{1,\cdots,K\}$, let $ \hat{P}_n^k \coloneqq \frac{1}{n}\left(\delta_{u_k} + \sum_{j=2}^{n}\delta_{\rz_j}\right)$. 
Using equations \ref{eq:dnu-cond}, \ref{eq:PB} and \ref{eq:PBcond} we may rewrite $\tilde D_\gamma$ as  
\begin{equation}
\label{eq:deltadisc}
\dt= \sum_{k=1}^{K}p_{k}\flg\underset{B}{\sup}\left(\fgl \E[G(\hat{P}_n)(B)]-\E[G(\hat{P}^k_n)(B)]\right) .
\end{equation}
For any integer $n$, let $\gM_n$ be the set of all possible empirical distributions for data sets with $n$ points and let  $\gG_{n}=G(\gM_n)$. Since $P$ has at most countable support, then $\gG_{n}$ is at most countable and \eqref{eq:deltadisc}
gives
\begin{equation}
\label{eq:deltadiscBis}
\dt= \sum_{k=1}^{K}p_{k}\flg\underset{B}{\sup}\sum_{g\in\gG_n} g(B)\left[\fgl \mathbb{P}(G(\hat{P}_n)=g) - \mathbb{P}(G(\hat{P}_n^k)=g)\right].
\end{equation}

\noindent For some fixed $g\in\gG_n$, let us denote by $\gM_{n}(g)=G^{-1}(\{g\})\cap \gM_{n}$ the set of possible empirical distributions $Q$ in $\gM_{n}$ such that $G(Q)=g$. Then we have for any $g\in\gG_n$,
$$
g(B)\left(\fgl \mathbb{P}(G(\hat{P}_n)=g) - \mathbb{P}(G(\hat{P}_n^k)=g)\right)=\sum_{Q\in\gM_{n}(g)}
G(Q)(B)\left(\fgl P^n(\hat{P}_n=Q) - P^n(\hat{P}_n^k=Q)\right),
$$

\noindent so that summing over all $g$ gives

\begin{equation}
\label{eq:sumsupsum}
\dt= \sum_{k=1}^{K}p_{k}\flg\underset{B}{\sup}\sum_{Q\in\gM_n}G(Q)(B)\left[\fgl P^n(\hat{P}_n=Q) - P^n(\hat{P}_n^k=Q)\right],
\end{equation}

\noindent since $\left(\gM_{n}(g)\right)_{g\in \gG_n}$ is a partition of $\gM_n$.
As the distribution is discrete, any possible value $Q$ of $\hat{P}_n$ is uniquely determined by a $K-$tuple $(k_1,\cdots,k_K)$ (if $K=\infty$ then by a sequence $(k_1,k_2,\cdots)$) of non-negative integers such that $\sum_{j=1}^{K}k_j = n$ and $Q= \frac{1}{n}\sum_{j=1}^{K}k_j\delta_{u_j}$. The $K-$tuple (or sequence) corresponds to the distribution of the samples among the atoms, that is, if we define, for $j=1,\ldots,K$, the random variable $N_{j}$
as the number of samples in the dataset equal to $u_{j}$, then
for such $Q$,
$$
P^n(\hat{P}_n = Q) = \mathbb{P}(N_{j}=k_{j};\;j=1,\ldots,K).
$$

\noindent Since the samples are i.i.d., we get for such $Q$
\begin{equation}
\label{eq:binomdisc}
 P^n(\hat{P}_n = Q) = \binom{n}{k_1,\cdots ,k_K}\prod_{j=1}^{K}{p_j}^{k_j},
\end{equation}
where $ \binom{n}{k_1,\cdots, k_m} = \frac{n!}{k_1!\cdots k_m!}$ is the multinomial coefficient. Notice that when $K=+\infty$, only a finite number $m$ of integers $k_j$ are non zero, so that \eqref{eq:binomdisc} can be understood to hold also when $K=+\infty$ by keeping only the terms involving the positive integers $k_j$.
\\

\noindent Let us now compute $P^{n-1}(\hat{P}_n^1=Q)$. If $k_{1}=0$, then
$P^{n-1}(\hat{P}_n^1=Q)=0$. Else,

\begin{align*}
P^{n-1}(\hat{P}_n^1=Q) &=\mathbb{P}(N_{1}=k_{1}-1,\;N_{j}=k_{j};\;j=2,\ldots,K)\\
  &= \binom{n-1}{k_1-1, k_2,\cdots ,k_K}\left(\prod_{j=2}^{K}p_j^{k_j}\right)p_1^{k_1-1} \\ 
    &= \frac{k_1}{np_1}\binom{n}{k_1,\cdots ,k_K}\prod_{j=1}^{K}p_j^{k_j},
\end{align*}
which again is understood to hold also when $K=+\infty$.\\
Therefore in both cases, we get 
\begin{equation}
\label{eq:binomdisc2}
P^{n-1}(\hat{P}_n^1=Q) = \frac{k_1}{np_1}\binom{n}{k_1,\cdots ,k_K}\prod_{j=1}^{K}p_j^{k_j}.
\end{equation}

\noindent Now, using \ref{eq:binomdisc} and \ref{eq:binomdisc2}, denoting by $g_N$ the image by $G$ of the distribution determined by the $K-$tuple ${N=(k_1,\cdots,k_K)}$, we get
\begin{eqnarray*}
\sum_{Q\in\gM_n}G(Q)(B)\left[\fgl P^n(\hat{P}_n=Q) - P^{n-1}(\hat{P}_n^1=Q)\right]&=&\sum_{k_1+\cdots+k_K=n}g_N(B)\binom{n}{k_1,\cdots, k_K}\prod_{j=1}^{K}p_j^{k_j}\left(\fgl-\frac{k_1}{np_1}\right) \\  
    &=& \E\left[\left(\fgl-\frac{N_1}{np_1}\right)g_N(B)\right],  
\end{eqnarray*}

\noindent where $N = (N_1,\cdots,N_K)$ follows a multinomial distribution of parameters $(n;p_1,\cdots,p_K)$. The computation being similar for any $k=1,\ldots, K$, we easily obtain,
$$
\sum_{Q\in\gM_n}G(Q)(B)\left[\fgl P^n(\hat{P}_n=Q) - P^{n-1}(\hat{P}_n^k=Q)\right]=\E\left[\left(\fgl- \frac{N_k}{np_k}\right)g_N(B)\right].
$$

\noindent Now, plugging it into \eqref{eq:sumsupsum} gives,

\begin{equation}
\label{eq:sumsup}
\dt= \sum_{k=1}^{K}p_{k}\flg\underset{B}{\sup}\E\left[\left(\fgl- \frac{N_k}{np_k}\right)g_N(B)\right].
\end{equation}

\noindent We get from \eqref{eq:sumsup}

\begin{align}
\begin{split}
\label{eq:sumplus}
\dt&\leq \sum_{k=1}^{K}p_{k}\flg\E\left[\underset{B}{\sup}\left(\fgl- \frac{N_k}{np_k}\right)g_N(B)\right] \\ 
&= \sum_{k=1}^K p_k\flg\E\left[\left(\fgl- \frac{N_k}{np_k}\right)_+\right]
\end{split} \\ 
\label{eq:summinus}
\dt&\leq \sum_{k=1}^K p_k\flg\E\left[\left(\fgl- \frac{N_k}{np_k}\right)_-\right] + \left(1-\flg\right),
\end{align}

\noindent where the equality in \eqref{eq:sumplus} comes from the fact that the supremum is reached on null sets when $\fgl-\frac{N_k}{np_k}$ is negative, and on sets of mass 1 when it is positive. Equation \ref{eq:summinus} is obtained by replacing $B$ by its complementary $B^c$ in the supremum and remarking that $\E[\fgl -\frac{N_k}{np_k}]=\fgl-1$. Taking the average of equations \ref{eq:sumplus} and \ref{eq:summinus} gives 
\begin{align*}
    \dt &\leq \frac{1}{2\gamma}\sum_{k=1}^{K}\E\left[\left|\frac{N_k}{n} - \fgl p_k\right|\right] + \frac{1}{2}\left(1-\flg\right),
\end{align*}
\noindent hence we get
\begin{align*}
    \dnu &\leq \max(1,\gl)\left[\frac{1}{2\gamma}\sum_{k=1}^{K}\E\left[\left|\frac{N_k}{n} - \fgl p_k\right|\right] + \frac{1}{2}\left(1-\flg\right) - \lb\right] \\ 
    &= \max(1,\gl)\left[\frac{1}{2\gamma}\sum_{k=1}^{K}\E\left[\left|\frac{N_k}{n} - \fgl p_k\right|\right] - \frac{1}{2}\left|1-\flg\right|\right],
\end{align*}

\noindent which proves Equation \ref{eq:deltabornesup}.

%% file: Proofs/Lemma-discretedeltamax.tex
For some fixed $\theta\in\Theta$, we similarly denote by $\gM_n(\theta) = F^{-1}(\{\theta\})\cap\gM_n$ the set of possible empirical distributions $Q$ in $\gM_n$ such that $F(Q) = \theta$. Using \eqref{eq:dnu-cond}, \eqref{eq:dnu-int}, and following similar steps as in equations \ref{eq:deltadisc}, \ref{eq:deltadiscBis} and \ref{eq:sumsupsum}, by triangular inequality, we get that 

\begin{align}
\notag
\dt &= \flg\sum_{k=1}^{K}\frac{p_k}{2}\sum_{g\in\gG_n}\left|\fgl \mathbb{P}(\delta_{F(\hat{P}_n)}=g) - \mathbb{P}(\delta_{F(\hat{P}_n^k)}=g)\right|  + \frac{1}{2}\left(1-\flg\right)\\ 
\notag
&= \flg\sum_{k=1}^{K}\frac{p_k}{2}\sum_{\theta\in\Theta}\left|\fgl\mathbb{P}(F(\hat{P}_n)=\theta)-\mathbb{P}(F(\hat{P}_n^k)=\theta)\right| + \frac{1}{2}\left(1-\flg\right)\\ 
\label{eq:deltadisclem}
&= \flg\sum_{k=1}^{K}\frac{p_k}{2}\sum_{\theta\in\Theta}\left|\sum_{Q\in\gM_n(\theta)}\left(\fgl P^{n}(\hat{P}_n=Q)-P^{n-1}(\hat{P}_n^k=Q)\right)\right| + \frac{1}{2}\left(1-\flg\right) \\ 
\notag
&\leq \flg\sum_{k=1}^{K}\frac{p_k}{2}\sum_{Q\in\gM_n}\left|\fgl P^n(\hat{P}_n=Q) - P^{n-1}(\hat{P}_n^k=Q)\right| + \frac{1}{2}\left(1-\flg\right),
\end{align}

\noindent since $(\gM_n(\theta))_{\theta\in\Theta}$ is a partition of $\gM_n$. We now prove that when taking the maximum over all possible measurable maps $F$ having range $\Theta$, the inequality becomes an equality. Indeed, since $\Theta$ is infinite, it is possible to construct $F$ such that $F$ is an injection from $\bigcup_{n\in\mathbb{N}}\gM_n$ to $\Theta$, in which case for all $\theta\in\Theta$, $\gM_n(\theta)$ is either the empty set or a singleton. Thus, \eqref{eq:deltadisclem} gives

$$\underset{F}{\max}\dt = \frac{1}{2\gamma}\sum_{k=1}^{K}p_k\sum_{Q\in\gM_n}\left|\fgl P^n(\hat{P}_n=Q) - P^{n-1}(\hat{P}_n^k=Q)\right| + \frac{1}{2}\left(1-\flg\right),$$

\noindent and the lemma follows from equations \ref{eq:binomdisc} and \ref{eq:binomdisc2} and the same steps as in the proof of Theorem \ref{theo:disc}.

%% file: Proofs/Corollary-discrete.tex
By Theorem \ref{theo:disc}, we have
\begin{align*}
    \underset{\gA}{\max}\;\dnu &=  \max(1,\gl)\left[\frac{1}{2\gamma}\sum_{k=1}^{K}\E\left[\left|\frac{B_k}{n} - \fgl p_k\right|\right] - \frac{1}{2}\left|1-\flg\right|\right].
\end{align*}
\noindent Now, from Lemma \ref{lem:corcp-1}, for any $\gamma\leq1$, we have $\lg\geq1$, and it holds that 
\begin{align*}
   \underset{\gA}{\max}\;\dnu &= \frac{1}{2}\sum_{k=1}^{K}\E\left[\left|\frac{1}{\gamma}\frac{B_k}{n}-p_k\right|\right] - \frac{1}{2}\left(\flg-1\right) \\&= \frac{1}{2}f_+(\lg) \\
    &\leq \frac{1}{2}f_+(1)\\
    &= \frac{1}{2}\sum_{k=1}^{K}\E\left[\left|\frac{B_k}{n}-p_k\right|\right].
\end{align*}
\noindent From Lemma \ref{lem:corcp-2}, for any $\gamma\geq1$, we have $\lg\leq1$, and it holds that
\begin{align*}
    \underset{\gA}{\max}\;\dnu &= \frac{1}{2}\sum_{k=1}^{K}\E\left[\left|\frac{B_k}{n}- \fgl p_k\right|\right] - \frac{1}{2}\left(\fgl-1\right) \\ 
    &= \frac{1}{2}\sum_{k=1}^{K}\mathbb{E}\left[\fgl\left|\flg\frac{B_k}{n}- p_k\right| - p_k\left(\fgl-1\right)\right] \\ 
    &=\frac{1}{2}\sum_{k=1}^{K}\mathbb{E}\left[f_{B_k/n, p_k}^-\left(\flg\right)\right] \\ 
    &\leq \frac{1}{2}\sum_{k=1}^{K}\mathbb{E}\left[f_{B_k/n, p_k}^-\left(1\right)\right] \\ 
    &= \frac{1}{2}\sum_{k=1}^{K}\E\left[\left|\frac{B_k}{n}-p_k\right|\right].
\end{align*}

\noindent In particular, this means that the minimum security is reached for $\gamma=1$, i.e. $\nu=1/(1+\lambda)$. On the one hand, if $\gamma=1$ and for any $k=1,\ldots,K$, $n>1/p_k$, then also for any $k=1,\ldots,K$,  $n > 1/(1-p_k)$. This implies that for any $k=1,\ldots,K$, we have $1\leq m_k \leq n-2$. Then Theorem \ref{theo:disc}, Lemma \ref{lem:corcp-3} and Lemma \ref{lem:corcp-4} yield the result.\\
\noindent On the other hand, if the condition on $n$ does not hold, we still have for any $k=1,\ldots,K$ by Cauchy-Schwartz inequality,

$$\mathbb{E}\left[\left|\frac{B_k}{n}-p_j\right|\right] \leq \sqrt{Var(B_k/n)} = n^{-1/2}\sqrt{p_k(1-p_k)},$$
\noindent which shows the lower bound on the MIS when the condition on $n$ does not hold..

%% file: Proofs/Lemma-cp1.tex
\noindent Let $q_j^k = \mathbb{P}(B_k=j)$. Then we have 
$$f^+(x)=\sum_{k=1}^{K}\sum_{j=1}^{n}q_j^k\left|x\frac{j}{n}-p_k\right| - (x-1).$$
Now note that $f^+$ is continuous and almost everywhere differentiable. Letting $m_j=\lfloor np_j/x\rfloor$, on any point of differentiability, we have,
\begin{align*}
    (f^+)'(x) &= \sum_{k=1}^{K}\sum_{j=1}^{n}q_j^k\frac{\partial}{\partial x}\left| x\frac{j}{n}-p_k\right| - 1 \\ 
    &=- 1 +  \sum_{k=1}^{K}\sum_{j=1}^{n}q_j^k \begin{cases}
        -j/n & \text{, if }j<np_k/ x \\
        j/n & \text{, if }j>np_k/ x
    \end{cases} \\ 
    &= -1 + \sum_{k=1}^{K}\sum_{j=1}^{n}q_j^k\left[-\frac{j}{n}1_{j\leq m_k} + \frac{j}{n}1_{j>m_k+1}\right] \\ 
    &= -1 + \sum_{k=1}^K\left[-2\mathbb{E}\left[\frac{B_k}{n}1_{B_k\leq m_k}\right]+p_k\right] \\ 
    &=-2\sum_{k=1}^K\mathbb{E}\left[\frac{B_k}{n}1_{B_k\leq m_k}\right] \leq 0,
\end{align*}

\noindent which shows that $f^+$ is non-increasing by continuity of the function.

%% file: Proofs/Lemma-cp2.tex
The proof is similar to the proof of Lemma \ref{lem:corcp-1}. In particular, on any point of differentiability, we have

\begin{align*}
    (f_{t,p}^-)'( x) &= \frac{-1}{ x^2}| x t - p| + \frac{1}{ x}\frac{\partial}{\partial x}| x t - p| + \frac{p}{ x^2} \\ 
    &= \frac{1}{ x^2}\left[p - | x t - p|\right] +  x\begin{cases}
        -t & \text{, if }  x < p/t \\ 
        t & \text{, if }  x > p/t
    \end{cases} \\ 
    &= \frac{1}{ x^2}\begin{cases}
        p - (p- x t) -  x t & \text{, if }  x < p/t \\
        p - ( x t - p) +  x t\text{, if }  x > p/t
    \end{cases} \\ 
    &= \frac{2p}{ x^2}1_{ x>p/t} \geq 0,
\end{align*}

\noindent which shows that $f_{t,p}^-$ is non-decreasing, by continuity of the function.

%% file: Proofs/Lemma-cp3.tex
\noindent By Lemma \ref{lem:corcp-5}, for any $m_k\leq n$, we have $\mathbb{E}[(np_k-B_k)1_{B_k\leq m_k}] = \psi(m_k,p_k)$. Additionally, using the generalization $\binom{n}{n+k}=0$ for $k>0$, the equality still holds using the fact that $\mathbb{E}[np_k-B_k]=0$. Then we have

\begin{align*}
    \mathbb{E}\left[\left|B_k - \fgl np_k\right|\right] &= \mathbb{E}\left[\left(\fgl np_k - B_k\right)\left(1_{B_k\leq m_k} - 1_{B_k>m_k}\right)\right] \\ 
    &=\mathbb{E}\left[\left(\fgl np_k - \fgl B_k\right)\left(1_{B_k\leq m_k} - 1_{B_k>m_k}\right)\right]  + \left(\fgl-1\right)\mathbb{E}\left[B_k\left(1_{B_k\leq m_k} - 1_{B_k>m_k}\right)\right] \\ 
    &= 2\fgl\mathbb{E}\left[\left( np_k -  B_k\right)1_{B_k\leq m_k}\right] + \left(\fgl-1\right)\mathbb{E}\left[2B_k1_{B_k\leq m_k} - np_k\right] \\
    &= 2\fgl\psi(m_k,p_k) + \left(\fgl-1\right)\mathbb{E}\left[2(B_k-np_k)1_{B_k\leq m_k} + 2np_k\mathbb{P}(B_k\leq m_k) - np_k\right] \\ 
     &= 2\fgl\psi(m_k,p_k) + \left(\fgl-1\right)\left[-2\psi(m_k, p_k) + 2np_k\mathbb{P}(B_k\leq m_k) - np_k\right] \\ 
     &= 2\psi(m_k,p_k) + \left(\fgl-1\right)np_k\left(2\mathbb{P}(B_k\leq m_k)-1\right),
\end{align*}

\noindent which gives,
\begin{align}
   \nonumber \frac{1}{2\gamma} \sum_{k=1}^{K}\E\left[\left|\frac{B_k}{n}-\fgl p_k\right|\right] &= \sum_{k=1}^{K}\left(\frac{1}{\gamma n} \psi(m_k,p_k) + \frac{1}{2}\left(1-\flg\right)p_k \left(2\mathbb{P}(B_k\leq m_k)-1\right)\right)\\ 
   \label{eq:eq-psi} &= \frac{1}{\gamma n}\sum_{k=1}^{K}\psi(m_k,p_k) + \left(1-\flg\right)\left[-1/2 + \sum_{k=1}^{K}p_k\mathbb{P}(B_k\leq m_k)\right].
\end{align}

\noindent By Hoeffding Inequality, we have that, 
\begin{align*}
    \sum_{k=1}^{K}p_k\mathbb{P}(B_k\leq m_k) & \begin{cases}
    \geq  1-\sum_{k=1}^{K}p_ke^{-np_k^2(1-\gl)^2}&\text{ , if }\gamma>1 \\
    \leq  \sum_{k=1}^{K}p_ke^{-np_k^2(1-\gl)^2}&\text{  
 , if }\gamma\leq1
\end{cases},
\end{align*}

\noindent which gives 

\begin{align*}
     \frac{1}{2\gamma} \sum_{k=1}^{K}\E\left[\left|\frac{B_k}{n}-\fgl p_k\right|\right] &= \frac{1}{\gamma n}\sum_{k=1}^{K}\psi(m_k,p_k) + \left(1-\flg\right) \begin{cases}
    1/2 -  o(e^{-n}) &\text{, if }\gamma>1 \\
    -1/2 +  o(e^{-n})&\text{, if }\gamma\leq1
    \end{cases} \\
    &=  \frac{1}{\gamma n}\sum_{k=1}^{K}\psi(m_k,p_k)+ \frac{1}{2}\left|1-\flg\right| - \left|1-\flg\right|o(e^{-n}),
\end{align*}
\noindent hence the result.

%% file: Proofs/Lemma-cp4.tex
\noindent We shall bound $\psi(m,p)$ by \cite{stirling1955}, which states that for any integer $k\geq1$, it holds that

\begin{equation}
\label{eq:stirling}
\sqrt{2\pi}k^{k+1/2}e^{-k}e^{1/12(k+1)}<k!<\sqrt{2\pi}k^{k+1/2}e^{-k}e^{1/12k}.
\end{equation}

\noindent Set $a := \exp\left(\frac{1}{12(n+1)}-\frac{1}{12(m+1)}-\frac{1}{12(n-(m+1))}\right)$. Since $1\leq m\leq n-2$, we have


\begin{equation}
\label{eq:ak}
a\geq \exp(-1/6).
\end{equation}

\noindent One may apply Inequality \ref{eq:stirling} to get

\begin{align*}
    \binom{n}{m+1} &>\frac{\sqrt{2\pi}n^{n+1/2}}{\sqrt{2\pi}(m+1)^{m+1+1/2}\sqrt{2\pi}(n-(m+1))^{n-(m+1)+1/2}}\frac{e^{-n}}{e^{-(m+1)}e^{-(n-(m+1))}}a \\ 
    &= \frac{\sqrt{n}}{\sqrt{2\pi}}n^n\left[m+1\right]^{-(m+1+1/2)}\left[n-(m+1)\right]^{-(n-(m+1) + 1/2)}a \\ 
    &\coloneqq ba.
\end{align*}

\noindent Now,
\begin{align*}
    b(m+1)p^{m+1}(1-p)^{n-m} &=\frac{\sqrt{n}}{\sqrt{2\pi}}n^n\left[m+1\right]^{-(m+1+1/2-1)}\left[n-(m+1)\right]^{-(n-(m+1) + 1/2)}\\ 
    &\times p^{m+1}(1-p)^{n-m} \\ 
    &= \frac{\sqrt{n}}{\sqrt{2\pi}}n^n (np)^{-(m+1/2)} \left[\frac{m+1}{np}\right]^{-(m+1/2)} \\ 
    &\times (n(1-p))^{-(n-(m+1/2))} \left[\frac{n-(m+1)}{n(1-p)}\right]^{-(n-(m+1/2))} \\ 
    &\times p^{m+1}(1-p)^{n-m} \\
    &= \frac{\sqrt{n}}{\sqrt{2\pi}}\sqrt{p(1-p)}\left[\frac{m+1}{np}\right]^{-(m+1/2)} \\ 
    &\times\left[\frac{n-(m+1)}{n(1-p)}\right]^{-(n-(m+1/2))} \\ 
    &\coloneqq \frac{\sqrt{n}}{\sqrt{2\pi}}\sqrt{p(1-p)}d,
\end{align*}

\noindent which finally implies

$$\frac{\sqrt{n}}{\sqrt{2\pi}}\sqrt{p(1-p)}da<\binom{n}{m+1}(m+1)p^{m+1}(1-p)^{n-m}.
$$

\noindent Similarly, set $\tilde{a}\coloneqq \exp\left(\frac{1}{12n}-\frac{1}{12(m+2)} - \frac{1}{12(n-m)}\right)$. We then get 

\begin{equation}
\label{eq:tildea}
    \tilde{a} \leq e^{-1/36}.
\end{equation}

\noindent From the same steps as for the lower bound, we have,

$$\binom{n}{m+1}(m+1)p^{m+1}(1-p)^{n-m}<\frac{\sqrt{n}}{\sqrt{2\pi}}\sqrt{p(1-p)}d\tilde{a}.
$$

\noindent Define $\epsilon\in[0,1)$ such that $np=m+\epsilon$. Then
$$
d=\exp\left\{\left(m+\frac{1}{2}\right)\ln\left(\frac{m+\epsilon}{m+1}\right)
+\left(n-m-\frac{1}{2}\right)\ln\left(\frac{n-m-\epsilon}{n-m-1}\right)
\right\}.
$$
For any $m\in\{1,\ldots,n-2\}$ and $\epsilon\in[0,1)$, define
$$
f(m,\epsilon)=\left(m+\frac{1}{2}\right)\ln\left(\frac{m+\epsilon}{m+1}\right)
+\left(n-m-\frac{1}{2}\right)\ln\left(\frac{n-m-\epsilon}{n-m-1}\right).
$$
By studying the function $\epsilon\mapsto f(m,\epsilon)$ we get that for all  $\epsilon\in[0,1)$, 
$f(m,\epsilon)\geq \min\{f(m,0),0\}$. By studying the function $m\mapsto f(m,0)$ we get that
for all  $m\in\{1;\ldots;n-2\}$,
$f(m,0)\geq \min\{f(1,0),f(n-2,0)\}$. But
$$
f(1,0)=-f(n-2,0)=-\frac{3}{2}\log (2)+\left(n-\frac{3}{2}\right)\log\left(1+\frac{1}{n-2}\right).
$$
Now, Taylor expansion of $\log(1+u)$ allows to prove
$$
-\frac{3}{2}\log (2)+1-\frac{1}{4(n-2)^2}\leq f(1,0)
\leq -\frac{3}{2}\log (2)+1+\frac{1}{2(n-2)}.
$$
When $n\geq 5$, it is easy to see that $-\frac{3}{2}\log (2)+1-\frac{1}{4(n-2)^2}>0$, so that using \eqref{eq:ak}, and setting
$$c=\frac{\exp\left(\frac{3}{2}\log (2)-1-1/3\right)}{\sqrt{2\pi}},$$

\noindent we get,

$$\psi(m,p) \geq c\sqrt{n}\sqrt{p(1-p)}.$$

\noindent A rough approximation gives $c>0.29$. \\
\noindent By similar arguments, we obtain 
\begin{equation}
    \log(d) \leq \frac{7}{2}\log(7) - 2\log(3) - \frac{13}{2}\log(2) .
\end{equation}
\noindent Using \eqref{eq:tildea} and setting
$$c' = \frac{\exp( \frac{7}{2}\log(7) - 2\log(3) - \frac{13}{2}\log(2) - 1/36)}{\sqrt{2\pi}},$$
\noindent we get,
$$\psi(m,p) \leq c'\sqrt{n}\sqrt{p(1-p)}.$$
\noindent A rough approximation gives $c' < 0.44$.

%% file: Proofs/Lemma-cp5.tex
We will prove the result by recursion. First, for $m=0$, note that we have,
\begin{align*}
    \binom{n}{0}p^0(1-p)^n(np-0) &= np(1-p)^n\\
    (0+1)\binom{n}{1}p^{0+1}(1-p)^n &= np(1-p)^n,
\end{align*}
which proves the initial statement. Assume that the result holds for some $m < n$, then we have,
\begin{align*}
    \sum_{k=0}^{m+1}\binom{n}{k}p^k(1-p)^{n-k}(np-k) &= (m+1)\binom{n}{m+1}p^{m+1}(1-p)^{n-m} \\&+ \binom{n}{m+1}p^{m+1}(1-p)^{n-(m+1)}(np-(m+1)) \\ 
    &= \binom{n}{m+1}p^{m+1}(1-p)^{n-(m+1)}\left[(m+1)(1-p) + np - (m+1)\right] \\ 
    &= p^{m+2}(1-p)^{n-(m+1)}(n-(m+1))\frac{n!}{(m+1)!(n-(m+1))!} \\ 
    &= (m+2)\binom{n}{m+2}p^{m+2}(1-p)^{n-(m+1)},
\end{align*}
\noindent where the first equality comes from the recursion hypothesis.\\
\noindent By recursion, the lemma holds.

%% file: Appendix/v3/Y-DPrelation-v3.tex
We further discuss here the relation between differential privacy guarantees and the MIS studied in the paper. Specifically, we discuss the fact that guarantees given by differential privacy are not equivalent to guarantees given by the MIS. We first briefly describe differential privacy, and then compare it to the MIS.

\subsection{Differential Privacy}

For completeness, we first give the definition of $(\varepsilon,\delta)-$differential privacy.

\begin{definition}[$(\varepsilon,\delta)-$differential privacy \citep{dwork2014algorithmic}]
Let $\gA$ be a learning procedure. We say that $\gA$ satisfied $(\varepsilon,\delta)-$differential privacy if for any $S\subseteq \Theta$, and any neighboring datasets $D$ and $D'$, it holds

\begin{equation}
\label{eq:DP}
\mathbb{P}(\gA(D)\in S\mid D) \leq e^{\varepsilon}\mathbb{P}(\gA(D')\in S\mid D') + \delta.
\end{equation}
\end{definition}

\noindent In the previous definition, we say that $D$ and $D'$ are neighboring if they differ by at most one entry. For small values of $(\varepsilon,\delta)$, one shall interpret differential privacy as some stability measure of the learning procedure. Indeed, differential privacy states that the conditional distribution of $\gA(D)$ given $D$ is "almost" indistinguishable from the conditional distribution of $\gA(D')$ given $D'$.\\

\noindent In particular, Differential Privacy is a framework that considers the worst case scenario. Equation \ref{eq:DP} must be satisfied for any measurable set $S$ and any neighboring datasets $D$ and $D'$, including possible pathological datasets. In other words, Differential Privacy requires not to differentiate importance between datasets.\\

\noindent A key strength of differential privacy is its stability  under composition. More specifically, if a learning procedure $\gA$ is $(\varepsilon,\delta)-$DP, then it is $(k\varepsilon,k\delta)-$DP under $k-$fold composition \citep{dwork2006calibrating, dwork2009differential}. Future work have provided more precise characterization of the composition theorems \citep{kairouz2015composition}. One shall interpret these "composition theorems" as procedures to privatize a learning procedure. Given an $(\varepsilon,\delta)-$DP Mechanism $\mathbb{M}$ and a learning procedure $\gA$, one can use $\mathbb{M}\circ\gA$ to transform $\gA$ into an $(\varepsilon,\delta)-$DP learning procedure.

\subsection{Comparison to the MIS}
\label{subsec:compToMis}
\noindent We begin by discussing core conceptual differences between differential privacy and the MIS.\\
\noindent\textit{Differential privacy} is a tool to induce privacy into a model. Given a learning procedure $\gA$ and a known $(\varepsilon,\delta)-$DP mechanism $\mathbb{M}$, one can transform $\gA$ into an $(\varepsilon,\delta)-$DP learning procedure by composition $\mathbb{M}\circ\gA$. However, for an arbitrary learning procedure $\gA$, it does not provide a way of measuring the privacy of $\gA$.\\
\noindent \textit{The MIS} is a tool to measure the MIA-wise privacy of a learning procedure. For any learning procedure $\gA$ it measures its MIA-wise privacy. However, it does not provide a way to induce privacy into a model.\\
Consequently, differential privacy is a tool to induce privacy whereas the MIS is a metric to quantify the privacy. In particular, though the respective objectives might intersect, their utility are not equivalent. Differential privacy gives a broad framework to define and control privacy, which "\textit{intuitively will guarantee that a randomized learning procedure behaves similarly on similar input databases}" \citep{dwork2014algorithmic}. However, this does not formally answer a specific question, and shall have as many interpretations as there are privacy questions. On the converse, the MIS presented in this article gives a privacy metric specific to the MIA framework, arguably narrower than the differential privacy guarantees. The MIS is a metric that naturally comes into sight from the study of MIA accuracy, and results from considering the best performing MIA. Although this metric could be understood as quantifying some stability of the learning process, interpreting the MIS for other questions can not be done trivially.\\
\noindent In other words, differential privacy and the MIS does not convey the same message, and are simply to different frameworks.\\

\noindent Additionally, there are differences between their very definitions:

\begin{itemize}
    
    \item\textbf{Conditional vs Marginal.} For differential privacy to hold, \eqref{eq:DP} must be satisfied for any measurable set $S$ and any neighboring datasets $D$ and $D'$. Fundamentally, differential privacy does not differentiate between natural datasets and pathological datasets. On the contrary, the MIS integrates over all possible datasets, and therefore pathological datasets would influence negligibly the output metric.
    
    \item\textbf{Privacy of the Learning Procedure vs. Privacy of the Model.} 
    Given a dataset $\gD$, for differential privacy to hold, the relation

     $$\underset{S}{\sup}\underset{D'}{\sup}\left(\mathbb{P}(\gA(D)\in S\mid D) - e^\varepsilon\mathbb{P}(\gA(D')\in S\mid D')\right) \leq \delta,$$

     must hold for any neighboring datasets $\gD'$. Particularly, this shall be interpreted as any trained model $\mu_{\gA(\gD)}$ must be (almost) indistinguishable from a neighboring trained model $\mu_{\gA(\gD')}$. On the other hand, letting $D = \{\rz_1,\cdots,\rz_n\}$ and $D'=\{\rz_1',\rz_2,\cdots,\rz_n\}$, then up to affine transformation, the central statistical quantity can be written as (see \eqref{eq:dnu-cond}),
     $$\mathbb{E}_{\rz_1}\underset{S}{\sup}\mathbb{E}_{D'}\left[\mathbb{P}\left(\gA(D)\in S\mid D\right) - \flg\mathbb{P}\left(\gA(D')\in S\mid D'\right)\right].$$

    Requiring a similar upper bound $\dnu\leq\delta$ means that the learning procedure $\gA$ must be stable over the task $P$.
\end{itemize}

\noindent Although the paradigm of differential privacy differs from the one of the MIS, they still are comparable.\\

\noindent \textbf{From differential privacy to the MIS.} When the learning procedure is known to be $(\varepsilon,\delta)-$DP for some $\varepsilon\in\mathbb{R}^+$ and $\delta\in(0,1)$, then setting $D=\{\rz_1,\cdots,\rz_n\}$ and $D'=\{\rz_1',\rz_2,\cdots,\rz_n\}$, one can easily derive the equation of Remark \ref{remark:dp} by,

\begin{align}
    \label{eq:proof-remark8}\dt &= E_{\rz_1}\left[\finv\gamma\underset{S}{\sup} \;\gamma\mathbb{P}\left(\hat{\theta}_n\in S \mid \rz_1\right) - \mathbb{P}\left(\hat{\theta}_n\in S\right)\right] \\ 
  \nonumber  &= E_{\rz_1}\left[\finv\gamma\underset{S}{\sup} \mathbb{E}_{D'}\left[\;\gamma\mathbb{P}\left(\gA(D)\in S \mid D\right) - \mathbb{P}\left(\gA(D')\in S\mid D'\right)\right]\right] \\ 
   \nonumber & \leq E_{\rz_1}\left[\finv \gamma\underset{S}{\sup} \;\gamma\left(e^{\varepsilon}\mathbb{P}\left(\hat{\theta}_n\in S\right) + \delta\right) - \mathbb{P}\left(\hat{\theta}_n\in S\right)\right] \\ 
 \nonumber   &= \left(e^{\varepsilon} - \finv\gamma\right)_+ + \delta,
\end{align}
where the first equality comes from \eqref{eq:dnu-cond}. From that, we get

\begin{align*}
    \mis &\geq 
    1 - \max(1,\gl)\left(\left(e^{\varepsilon} - \flg\right)_+ + \delta - \lb\right).
\end{align*}

\noindent Notably, this simple relation shows that $\dnu$ can be controlled under differential privacy. It is worth noting that the bound we derived is heuristically similar to bounds already obtained for the KL divergence \citep{dwork2010boosting}. This bound could be potentially further refined by using Pinsker inequality and results on the reduction in KL divergence caused by passing data threough private channels (see Theorem 1 of \cite{duchi2018minimax} or Theorem 3.1 of \cite{amorino2023minimax}).    Additionally, in many contexts, contraction results have been proved in the differential privacy framework. A remarkable result discussed in Remark I of \cite{ghazi2024total} states that letting $\mathbb{A}_{\varepsilon,\delta}\coloneqq \{\gA: \gA\text{ is }(\varepsilon,\delta)\text{-DP}\}$ the set of all $(\varepsilon,\delta)-$DP learning procedures, when $\gamma=1$,  it holds that ${\sup}_{\gA\in\mathbb{A}_{\varepsilon,\delta}}\dnu = \delta + (1-\delta)\texttt{tanh}(\varepsilon/2)$ where $\texttt{tanh}$ is the hyperbolic tangent. This result still holds for Local Differentially Private (LDP) learning procedures \citep{dwork2006calibrating}, namely letting $\mathbb{A}_{\varepsilon}\coloneqq \{\gA: \gA\text{ is }\varepsilon\text{-LDP}\}$ be the set of all $\varepsilon-$ LDP learning procedures, then when $\gamma=1$ it holds that ${\sup}_{\gA\in\mathbb{A}_\varepsilon}\dnu = \texttt{tanh}(\varepsilon/2)$ \citep{kairouz2016extremal}.\\
\noindent When $\gamma\not= 1$, one has the upper bound $\dnu\leq \max(1,\gamma)\|\mathbb{P}_{(\hat{\theta}_n,\rz_0)}-\mathbb{P}_{(\hat{\theta}_n,\rz_1)}\|_{TV}$. This means that if $\gA\in\mathbb{A}_{\varepsilon}$, then we have $\dnu \leq \max(1,\gamma)\texttt{tanh}(\varepsilon/2)$ (and similarly for $\gA\in\mathbb{A}_{\varepsilon,\delta}$).

\noindent \textbf{From MIS to differential privacy.} On the converse, we discuss here that differential privacy is not required to hold relevantly for a learning procedure to be secure against MIAs. When the MIS is known to have some value $\text{Sec}_n(P,\gA) = s$, then it does not necessarily imply that there exist convenient values $(\varepsilon,\delta)$ for which the learning procedure $\gA$ is $(\varepsilon,\delta)-$DP. For instance, differential privacy does not extend nicely to deterministic learning procedures.
\begin{lemma}
\label{lem:DP-deter}
Assume $\gA$ is deterministic, then for any finite $\varepsilon \geq 0$ and $\delta\in(0,1)$, $\gA$ is not $(\varepsilon,\delta)-$DP.    
\end{lemma}

\noindent However, as Theorem \ref{theo:disc} together with Corollary \ref{cor:cp} show that for any learning procedure $\gA$, we have $\mis\underset{n\to\infty}{\to} 1$. \\
\noindent Another result which does not require the learning procedure to be deterministic can be stated. Assume the hypotheses of Section \ref{subsec:emp-mean}. Let $L:\gZ\to\mathbb{R}^d$ and $F:\mathbb{R}^d\to\mathbb{R}^q$. Hence, we consider that the learning procedure outputs 

$$\hat{\theta}_n = F\left(\frac{1}{n}\sum_{j=1}^n L(\rz_j)\right).$$

\noindent Further assume that the support of $\rz$ is the whole space. Additionally, we may assume that the image of $L$, and $F$ is the whole space as well.

\begin{lemma}
\label{lem:DP-normal}
Assume that $F(x) = \tilde{F}(x) + \eta N$, for some deterministic function $\tilde F$ and $N$ follows a Normal distribution. Then for any finite $\varepsilon\geq0$ and any $\delta\in(0,1)$, $\gA$ is not $(\varepsilon,\delta)-$DP.
\end{lemma}

\noindent Theorem \ref{theo:empmean} and Lemma \ref{lem:DP-normal} show that differential privacy can not provide relevant information, whereas the MIS can.\\

\begin{proof}[Proof of Lemma \ref{lem:DP-deter}] Let fix $D$ and $D'$. As $\gA$ is assumed to be deterministic, we have that $\hat\theta = \gA(D)$ and $\hat\theta'=\gA(D')$ are constants conditionally to $D$ and $D'$. As the differential privacy property must hold for any measurable set $S$, it must hold for singletons $\{\theta\}\subseteq\Theta$. If $\hat\theta\not=\hat\theta'$, then there exists $S$ such that $\mathbb{P}(\hat\theta\in S\mid D) = 1$ and $\mathbb{P}(\hat\theta'\in S\mid D') = 0$. Hence, $\gA$ can not be $(\varepsilon,\delta)-$DP for any $\delta<1$.
\end{proof}

\begin{proof}[Proof of Lemma \ref{lem:DP-normal}]
Let fix $D$ and $D'$. Let $m = \tilde F\left(\frac{1}{n}\sum_{\rz\in D}L(\rz)\right)$ and $m' = \tilde 
 F\left(\frac{1}{n}\sum_{\rz\in D'}L(\rz)\right)$. Conditionally to $D$ (resp. $D'$), $\gA(D)$ (resp. $\gA(D')$) follows a Normal distribution with mean $m$ (resp. $m'$) and with covariance matrix $\eta^2 I_q$. Let $S_t = \{v\in\mathbb{R}^q : v_1 \leq t\}$. Then for $\gA$ to satisfy $(\varepsilon,\delta)-$DP, it must hold, for any $t\in\mathbb{R}$, that 

\begin{equation}
\label{eq:dp-St}
    \mathbb{P}\left(\gA(D)\in S_t\mid D\right) \leq e^{\varepsilon}\mathbb{P}\left(\gA(D')\in S_t\mid D'\right) + \delta.
\end{equation}

\noindent It is easy to verify that $\mathbb{P}\left(\gA(D)\in S_t\mid D\right) = \Phi\left((t-m_1)/\eta\right)$ and $\mathbb{P}\left(\gA(D')\in S_t\mid D'\right) = \Phi\left((t-m_1')/\eta\right)$, where $\Phi$ is the cumulative distribution function of the Standard Gaussian distribution. Now, as \eqref{eq:dp-St} must be satisfied for any neighboring datasets $D$ and $D'$, it must be satisfied when swapping roles between $D$ and $D'$. Therefore, without loss of generality, we assume that $m'_1\geq m_1$. From the fact that the support of $\rz$ is the whole space and the image of $L$ and $F$ is the whole space as well, we can construct neighboring datasets $D$ and $D'$ for which $|m_1-m_1'|$ is arbitrarily large. In particular, we can consider $m_1$ arbitrarily large negatively and $m_1'$ arbitrarily large positively which gives for any $t\in\mathbb{R}$,

$$\underset{m_1,m_1'}{\sup}\left|\Phi\left(\frac{t-m_1}{\eta}\right) - \Phi\left(\frac{t-m_1'}{\eta}\right)\right| = 1.$$

\noindent In particular, this means that there can not exist a finite $\varepsilon$ and $\delta<1$ such that $\gA$ is $(\varepsilon,\delta)-$DP.

\end{proof}